\documentclass{article}
   \usepackage[pdftex]{graphicx}
\usepackage[cmex10]{amsmath}

\usepackage{mdwmath}
\usepackage{mdwtab}
\usepackage{amsmath}
\usepackage{amsfonts}
\usepackage{amssymb}
\usepackage{amsthm}

\usepackage{algorithm}
\usepackage{algorithmic}

\theoremstyle{definitions}
\newtheorem{theorem}{Theorem}[section]
\newtheorem{lemma}{Lemma}[section]
\newtheorem{proposition}{Proposition}[section]

\newtheorem{corollary}{Corollary}[section]

\newcommand\cip{{i\!p\!}^{\times}}
\newcommand\sign{\mathrm{sign}}
\newcommand\R{\mathbb{R}}
\newcommand\nor{\mathcal{N}}
\newcommand\V{\mathrm{v}}
\newcommand\il[1]{\langle #1 \rangle}

\newcommand\minusOne{{_{-}}}
\newcommand\plusOne{{_{+}}}
\newcommand\pmp{+\!-\!+}
\newcommand\dist{\mathrm{d}}
\newcommand\margin{\mathrm{M}}

\newcommand\svm{\mathrm{_{SVM}}}
\newcommand\RMC{MELC}
\newcommand\for{\text{ for }}

\usepackage{dsfont}

\usepackage{color}

\usepackage{MnSymbol}
\newcommand\de[1]{\lsem #1 \rsem}

\begin{document}
%
\title{Multithreshold Entropy Linear Classifier}
%
%
%
%

\author{Wojciech Marian Czarnecki,
        Jacek Tabor\\
\small{Faculty of Mathematics and Computer Science,}\\
\small{Jagiellonian Unviersity, Krakow, Poland.}\\
\small{\{wojciech.czarnecki, jacek.tabor\}@uj.edu.pl}
}
\maketitle

\begin{abstract}

Linear classifiers separate the data with a hyperplane. In this paper we focus on the novel method of construction of multithreshold linear classifier, which separates the data with multiple parallel hyperplanes. Proposed model is based on the information theory concepts -- namely Renyi's quadratic entropy and Cauchy-Schwarz divergence. 

We begin with some general properties, including data scale invariance. Then we prove that our method is a multithreshold large margin classifier, which shows the analogy to the SVM, while in the same time works with much broader class of hypotheses. What is also interesting, proposed method is aimed at the maximization of the balanced quality measure (such as Matthew's Correlation Coefficient) as opposed to very common maximization of the accuracy. This feature comes directly from the optimization problem statement and is further confirmed by the experiments on the UCI datasets.

It appears, that our Entropy Multithreshold Linear Classifier (\RMC) obtaines similar or higher scores than the ones given by SVM on both synthetic and real data. We show how proposed approach can be benefitial for the cheminformatics in the task of ligands activity prediction, where despite better classification results, \RMC{} gives some additional insight into the data structure (classes of underrepresented chemical compunds).

\end{abstract}

\section{Introduction}

Linear classifiers (SVM, perceptron, LDA, logistic regression) aim to find $v \in \R^d$ and $b \in \R$ such that the decision on the class of $x$ is based
on 
\begin{equation} \label{eq1}
\sign(v^Tx -b).
\end{equation}
The linear classification is important as it has
the advantage of small VC dimension. 
 The same ideas can be seen behind the neural networks (and their modifications like
Extreme Learning Machines~\cite{huang2011extreme} or Deep Learning~\cite{hinton2006fast}), where the activation
of the single neuron is given by \eqref{eq1}, while the role
played by it in the whole decision process is usually given by
%
\par
\medskip

{ \begin{tabular}{l}
STEP 1: calculate $v^Tx$, \\
STEP 2: make decision based on the sign of $v^Tx-b$.
\end{tabular}}
\medskip
\par
Although the linear classification is usually very efficient, even for the simple
sets in $\R$, like $\pmp$, see Figure~\ref{fig:pmp}, we
cannot obtain sufficient classification results. This
led to the need for kernelization procedure~\cite{cortes1995support}.

\begin{figure}[htb]
\begin{center}
\includegraphics[height=2.8cm]{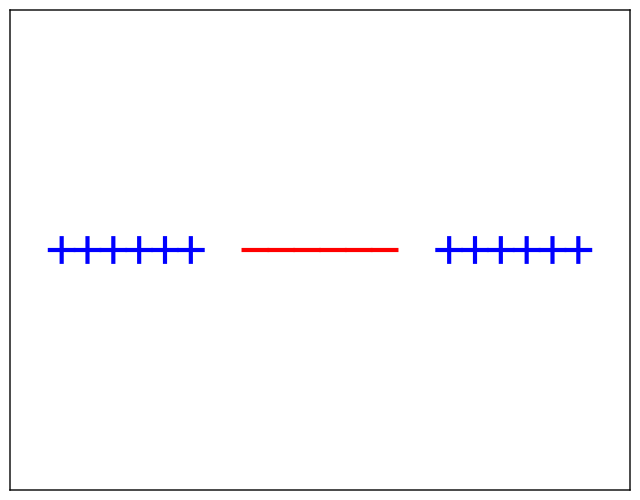}
\includegraphics[height=2.8cm]{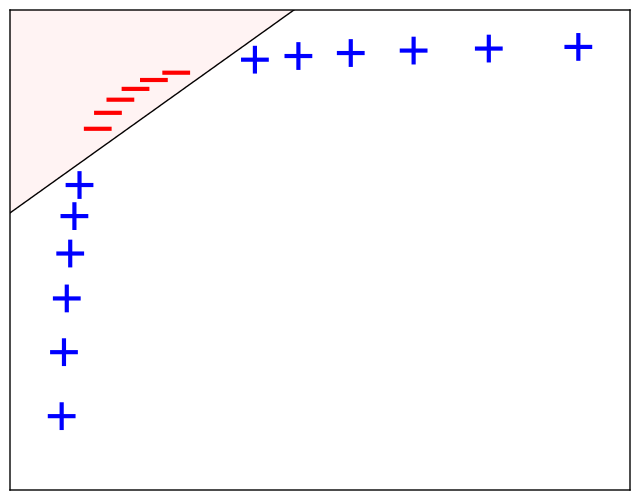}
\includegraphics[height=2.8cm]{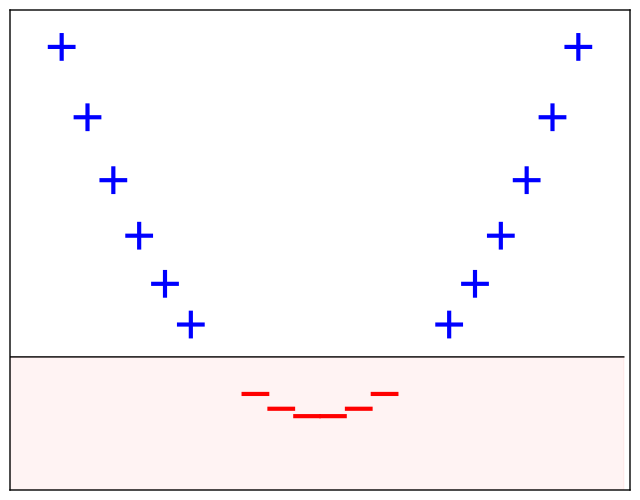}
\end{center}

\caption{From left: $\pmp$ dataset (linearly non-separable), $\pmp$ dataset in trained neural network with 2 hidden nodes with sigmoid activation functions, $\pmp$ dataset in trained SVM with polynomial kernel of degree 2}
\label{fig:pmp}
\end{figure}

Our postulate is that by applying the second step
we often lose some of the information given by the first
one -- observe that both in $\pmp$ or XOR case we 
can make sufficiently good classification decision 
based on the knowledge of the value of $v^Tx$ (for well
chosen $v$), see Figure~\ref{fig:toy}.
One can therefore ask why we do not use the
additional information? One of the possible answers
lies in the fact that most classification methods,
like SVM, aim at building a ``large'' linear margin between 
classes, which in a natural way leads to the single-threshold decision
boundary.

Thus there appears a natural question if we
can construct a classification method which would
 find the projection 
$x \to v^T x$ which could directly deal with
more complex classification cases like $\pmp$ and XOR.
The problem in fact splits into two -- how to
find the right $v \in \R^d$ and how to make the proper
classification decision in $\R$. The answer for the second
question is given by multithreshold linear classifiers~\cite{cao1994comparative},
where instead of decision based on the split of $\R$ into $(-\infty,b)$
and $[b,\infty)$ the division into finite number of 
intervals is allowed\footnote{This type of classification can be obtain in particular by the density based classifiers in $\R$.}.

The answer to the first question is nontrivial, and in our opinion
there could be many reasonable solutions. In this paper
we have decided to base the
decision on entropy-based divergence measure~\cite{principe2000information}. We have
chosen the {\em Renyi's quadratic entropy} 
$$
H_2(f):=-\log \int f^2
$$
and the connected {\em Cauchy-Schwarz divergence}
\begin{equation} \label{eq:DCS}
\begin{array}{l}
D_{CS}(f,g):=\log \int f^2 +\log \int g^2 -2\log \int fg \\[1ex]
=-\left (H_2(f)+H_2(g)+2 \log \cip(f,g) \right ),
\end{array}
\end{equation}
where $\cip(f,g):=\int fg$ denotes the {\em cross-information potential}. Our reasons behind such a
choice are the following:
\begin{itemize} 
\item Renyi entropy and the Cauchy-Schwarz divergence are easily computable and the exact formulas for the Gaussian mixtures are known (this allows
the use of gradient methods in our optimization problem, see Pracitcal Considerations Section),
\item the Cauchy-Schwarz divergence is translation and scale invariant in terms of input data transformation,
\item $D_{CS}$ has nice theoretical properties, as the minimization of 
$\cip$ leads to the maximization of the multi-threshold
boundary\footnote{To some extent we obtain multi-threshold analogue of large margin classifier.}, while the part consisting 
of Renyi's entropies adds the regularizing term, see Theory Section.
\end{itemize}

From the practical point of view, we first project the
data by $v^T$ onto $\R$, and apply there the classical kernel density
estimation given for the dataset $P \subset \R$ by
%
\begin{equation}
\de{P}_{\sigma}:=\frac{1}{|P|}\sum_{p \in P}\nor(p,\sigma^2),
\end{equation}
where $\nor(p,\sigma^2)$ denotes the one dimensional normal distribution.  We skip the subscript $\sigma$ (which denotes the window width) if it is chosen according to
the Silverman's rule~\cite{silverman1986density}
\begin{equation} \label{eq:sr}
\sigma=(4/3)^{1/5}|P|^{-1/5}\sigma_P, 
\end{equation}
where $\sigma_P$ denotes the standard deviation of the data $P$.
Only later we calculate the Cauchy-Schwarz divergence.
It is important to notice that our method performs density estimation in one-dimensional space $\R$. It is a common knowledge
that density estimation in high dimensions is unreliable (requries enormous amount of samples), which is one of the reasons
why purely density based classification is rarely used. In particular, even in the simplest case when data comes from multivariate
normal distribution and we are interested in the good estimation of the value at $0$, we need over $10,000$ samples for just 
7 dimensions~\cite{silverman1986density}. On the other hand, in one dimension we just need 4 samples (to obtain a solution with $0.1$ precision in terms of
mean squared error). 
This supports 
the idea behind creation of models based on 1-dimensional linear projections as they provide reliable estimation of the underlying 
densities.


Consequently our final optimization problem can be formulated as follows:
\medskip

\noindent{\bf Optimization problem.}
{\em
Consider classes $X_{\plusOne}$ and $X \minusOne$ in $\R^d$.
Find nonzero $v \in \R^d$ which maximizes the value of 
$$
D_{CS}(\de{v^TX_{\plusOne}},\de{v^TX_{\minusOne}}).
$$
}

The resulting multithreshold classifier is constructed from the 
density estimations $\de{v^TX_{\plusOne}}$ and $\de{v^TX_{\minusOne}}$.
Observe that, contrary to SVM, in our basic method we do not have any free parameters. 

As it is shown in the Evaluation Section, such model usually obtains similar or better classification quality than the linear SVM. It occurs that in practice due to the strong regularization proposed method selects quite small number of thresholds (which reduces the VC dimension~\cite{anthony2004generalization} of the resulting model). In fact, when using Silverman's rule for kernel window width estimation, our method built a single threshold model in nine out of ten UCI datasets. It is worth noting that these solutions are significantly different from the ones given by SVM so, even though their scores are similar, proposed method is fundamentally different and therefore gives additional knowledge of the problem.

The interesting practical applications of multithreshold model is the more detailed insight into data geometry. Let us consider the task of ligands activity prediction for given proteins (which is further described in the Evaluation Section). Figure~\ref{fig:mlc} shows results of kernel density estimation for one of the obtained models for cathepsin ligands \cite{smusz2013influence}. 
\begin{figure}[hbt]
 \begin{center}
 \includegraphics[width=0.6\textwidth]{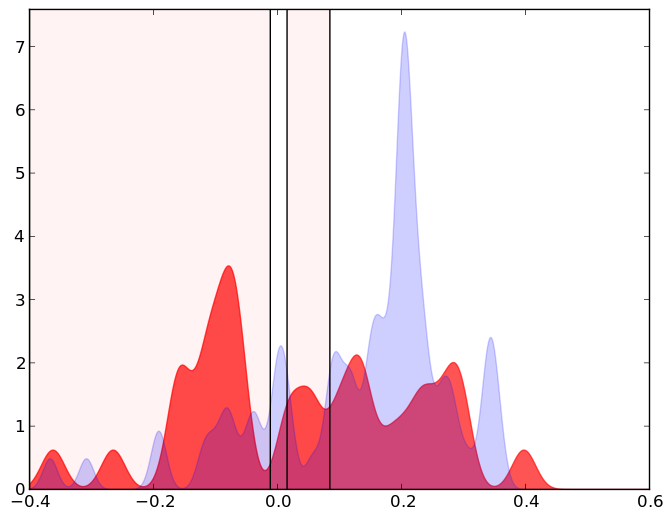}
 \caption{ Kernel density estimation of the linear projection of one of the folds of cathepsin ligands detection task using proposed multithreshold linear classifer for the test set. 
 The negative class spans through $x$ such that $v^Tx \in ( \infty, -0.02] \cup (0.02, 0.09]$ and the positive one through $x$ such that $v^Tx \in (-0.02,0.02] \cup (0.09,\infty) $}
 \label{fig:mlc}
  \end{center}
\end{figure}
One can notice how multithreshold classifier exploits the internal structure of the data by capturing small group of data points which is a part of the different class which would be ignored in linear classification. This results in the significant increase in the classification quality compared to the commonly encountered in this domain SVM model. This shows how the proposed model is able to exploit additional knowledge from the simple linear data projection. In the case of cheminformatics domain this typically represents some specific group of 
compounds\footnote{The main aim of this kind of research is identifying new drugs and2 compounds which are biologically active.}, distinctive from the most popular active ones (positive samples) and therefore is especially worth investigation\footnote{Exploiting such underrepresented groups of molecules might shed light on the currently under researched structural classes and lead to discovery of new types of drugs.}.

To sum up the \RMC{} (Multithreshold Entropy Linear Classifier) 
 has the following advantages:
\begin{itemize}
 \item has strong theoretical backgroud based on Information Theory,
 \item can build both single- and multithreshold linear classifiers,
 \item maximizes the balanced quality measure (is class imbalance invariant),
 \item is scale invariant (requires no data scaling),
 \item directly gives not only classification but also its likelihood (without the need for Platt's scaling),
 \item behaves well as the parameter-free model,
 \item although it tries to maximize the margins it builds significantly different model than SVM,
 \item can be parametrized to better fit data, and this free parameter has clear geometrical intuition.
\end{itemize}
Its current biggest drawback is computational complexity and existance of local solutions. 

Let us now briefly describe the contents of the paper. After short analysis of related work we show the basic properties of Cauchy-Schwarz divergence including its scale invariance and solutions for normaly distributed data. Next, we prove that proposed model maximizes the margins' sizes of multithreshold linear classifier and that the entropy terms play the regularization role. Then we proceed to some practical considerations regarding optimization procedure, its implementation and possible drawbacks. We conclude with the evaluation based on both synthethic and real datasets.

\section{Related work}



Multihreshold linear classifiers are present in machine learning for a long time \cite{takiyama1978multiple,olafsson1988capacity}, however they did not receive as much attention as the single threshold ones. One of the reasons may be hardness of their theoretical analysis and lack of answers for very basic question like their exact Vapnik-Chervonenkis dimension~\cite{anthony2003learning}. On the other hand Anthony et al. \cite{anthony2004generalization} recently showed some bounds regarding this class od models. However, efficient training of such models remains an open issue~\cite{anthony2003learning}.

As we will show, our method is strongly related to the Support Vector Machines concept~\cite{cortes1995support} or more generally large margin classifier idea~\cite{huang2008maxi,freund1999large,tipping2003relevance}. It is worth noting that we are not presenting a modification of SVM model (dozens of which appeared in recent years) but rather propose a conceptually different approach which leads to some important similarities.

Renyi's entropy has been deeply analyzed in the recent book by Principe et al.~\cite{principe2000information}, showing its wide applications spanning from classification optimization criterion~\cite{principe-class}, through clustering techniques~\cite{principe-clust} to ICA and other self-organizing methods~\cite{principe-self}. Use of Cauchy-Schwarz divergence for the classification criterion has been investigated in the past, in particular for a simple multilayer neural networks~\cite{santos2004error}. However, to the authors best knowledge, it has not yet been used as a criterion for the choice of one-dimensional linear projection used for density-based classification. 

In the broader sense, we are employing techniques from the information theory, which have been applied for construction of Decision Trees and, very successful model from 2001, Random Forest~\cite{breiman2001random}. On the other hand, density estimation based models have been recently used as the base of Deep Learning architectures~\cite{hinton2006fast} and proved to be a very good data processing technique.

  \section{Cauchy-Schwarz divergence}

In this section we discuss the basic theoretical aspects of the
the Cauchy-Schwarz divergence.
We show that the it is
insensitive to the change of scale, which consequently
yields that that we can restrict search to the unit
sphere $S:=\{v \in \R^d: \|v\|=1\}$. Next we discuss the case of normal distributions.


\subsection{Scale invariance}

We are going to show that the Cauchy-Schwarz divergence is scale invariant.
Observe that for $f,g:\R \to \R_+$ 
$$
D_{CS}(f,g)=-2 \log (\int \frac{f}{\|f\|_2} \frac{g}{\|g\|_2}),
$$
where $\|f\|_2$ denotes the $L^2$-norm of $f$. This implies that
$$
D_{CS}(\alpha f,\beta g)=D_{CS}(f,g) \text{ for }\alpha,\beta>0,
$$
which means that in the use of the Cauchy-Schwarz
divergence we do not have to normalize the data.

We show that $D_{CS}$ does not depend on the change of scale.
To do so we need the following notation:
for density $f$ in $\R$, we put 
$$
R_{\alpha}f(x):=\frac{1}{|\alpha|}f(x/\alpha).
$$
Observe that if $P \subset 
\R$ comes from the density $f$, then $\alpha P$ was generated from the density $R_{\alpha f}$. In other words the operation $R_\alpha$
corresponds (for the densities) to the operation of
rescaling the data by $\alpha\neq 0$.

\begin{lemma}
Consider densities $f,g$ in $\R$ and $\alpha \neq 0$. Then
$$
D_{CS}(f,g)=D_{CS}(R_\alpha f,R_\alpha g).
$$
\end{lemma}

\begin{proof}
One can easily see that
$$
\int R_\alpha h(x) R_\alpha\tilde h(x) dx=
\frac{1}{\alpha^2} \int h(x/\alpha) \tilde h(x/\alpha)dx
$$
$$
\stackrel{u=x/\alpha}{=}\frac{1}{|\alpha|}\int h(u)\tilde h(u)du.
$$
Applying the above we obtain that
$$
\begin{array}{l}
D_{CS}(R_\alpha f,R_\alpha g)\\
=\log \int (R_\alpha f)^2 +\log \int (R_\alpha g)^2-
2 \log \int R_\alpha f R_\alpha g\\
=\log \int f^2-\log |\alpha| +\log \int g^2-\log |\alpha| - 2 \log \int f g+2\log|\alpha|\\
=D_{CS}(f,g).
\end{array}
$$
\end{proof}

We obtain the following corollary as a direct consequence of the previous lemma and the fact that $\de{ \alpha P }_{\alpha r}=R_\alpha\de{P}_{r}$.

\begin{corollary}
Let $P_\plusOne,P_\minusOne \subset \R$ be given. Then
$$
D_{CS}(\de{\alpha P_\plusOne}_{\alpha r},\de{\alpha P_\minusOne}_{\alpha s})=D_{CS}(\de{P_\plusOne}_r,\de{P_\minusOne}_s).
$$
\end{corollary}

Since $\sigma_{\alpha P}=|\alpha| \sigma_P$, we obtain 
by the Silverman's rule \eqref{eq:sr} that 
$$
\de{\alpha P}=R_{\alpha} \de{P} \for P \subset \R.
$$
This implies that the Cauchy-Schwarz
divergence of the data projection does not depend on the rescaling of the data:
\begin{equation} \label{eq:rescal}
D_{CS}(\de{v^T(\alpha X_\plusOne)},\de{v^T(\alpha X_\minusOne)})=
D_{CS}(\de{v^T X_\plusOne},\de{v^TX_\minusOne})
\end{equation}
for $v \in \R^d$, $\alpha \neq 0$. Consequently, in its
maximization process we can restrict to the unit sphere.

Finally, we arrive at:

\begin{theorem} \label{th:scale}
Let $A:\R^d \to \R^d$ be a linear invertible map. Then

$$
\begin{array}{l}
\sup\{D_{CS}(\de{v^T (AX_\plusOne)},\de{v^T(AX_\minusOne)}):  v \in S\} \\[0.5ex]
=
\sup\{D_{CS}(\de{v^T X_\plusOne},\de{v^TX_\minusOne}):  v \in S\}.
\end{array}
$$
\end{theorem}

\begin{proof}
Let $v \neq 0$ be arbitrarily fixed and let $w=A^Tv$. Then 
$$
v^T(AX)=(A^Tv)^TX=w^TX,
$$
which implies that 
$$
D_{CS}(\de{v^T (AX_\plusOne)},\de{v^T(AX_\minusOne)}) 
=D_{CS}(\de{w^T X_\plusOne},\de{w^TX_\minusOne}).
$$
Dually, for an arbitrary $w \neq 0$ by putting $v=(A^{-1})^Tv$, we get
$$
D_{CS}(\de{w^T X_\plusOne},\de{w^TX_\minusOne}) 
=D_{CS}(\de{v^T (AX_\plusOne)},\de{v^T(AX_\minusOne)}).
$$
The assertion of the theorem follows directly from the above
inequalities and \eqref{eq:rescal}.
\end{proof}

It is easy to notice that analogously one can show that $D_{CS}$ is translation 
invariant.

\subsection{Data with Gaussian distribution}

We proceed to the 
case when the data was generated from the normal distribution.
Although in practice the datasets are discrete, we perform the calculations
on the original continuous distributions (in next section we obtain 
approximation of the densities which were used to generate the data
by gaussian kernel density estimation).

Let us recall that the multivariate normal density $\nor(m,\Sigma)$ 
in $\R^d$ with mean $m$ and covariance matrix $\Sigma$ is given by
$$
\nor(m,\Sigma)(x)=\frac{1}{(2\pi)^{d/2}(\det\Sigma)^{1/2}}
\exp(-\frac{1}{2}\|x-m\|^2_{\Sigma}),
$$
where $\|\cdot\|_{\Sigma}$ denotes the Mahalanobis norm given by
$\|x\|^2_{\Sigma}=x^T\Sigma^{-1}x$.

In our considerations we will use the following well-known~\cite{statystyka} formula
for the scalar product of two normal densities:
\begin{equation} \label{eq:scalar}
\int \nor(m_1,\Sigma_1) \nor(m_2,\Sigma_2)=
\nor(m_1-m_2,\Sigma_1+\Sigma_2)(0).
\end{equation}
Observe that from the above we easily conclude the 
value of the Renyi's quadratic entropy of the 
normal density:
\begin{equation} \label{eq:renyi}
\begin{array}{l}
H_2(\nor(m,\Sigma))=-\log(\nor(0,2\Sigma)(0)) \\[0.5ex]
\displaystyle{=\frac{d}{2}\log(4\pi)+\frac{1}{2}\log \det \Sigma.}
\end{array}
\end{equation}

\begin{theorem}
Let us consider the data $X$ which was generated from the normal density $\nor(m,\Sigma)$. Then the function 
$$
S \ni v \to H_2(v^T X)
$$
attains maximum for $v$ being the eigenvector corresponding to the 
maximal eigenvalue of $\Sigma$.
\end{theorem}

\begin{proof}
One can easily check that since $X$ has the density $\nor(m,\Sigma)$, the projection $v^TX$ of $X$ onto $\R$ has the density
$$
\nor(v^Tm,v^T\Sigma v),
$$
and therefore by \eqref{eq:renyi}
$$
H_2(v^T X)
=\frac{1}{2}\log(4\pi)+\frac{1}{2}\log(v^T\Sigma v)
\text{ for }v \in S.
$$
Consequently to maximize the Renyi's entropy we have
to maximize the value of $v^T\Sigma v$. To do so, let us
take as the base of $\R^d$ the orthonormal vectors $f_1,\ldots,f_d$ which diagonalize $\Sigma$, ordered decreasingly according to the eigenvalues 
$\lambda_1 \geq \ldots \geq \lambda_d$ of $\Sigma$. Then
clearly 
\begin{equation} \label{eq:constr}
v^T\Sigma v=\sum_{i=1}^d \lambda_i v_i^2, 
\end{equation}
where $v$ has the coefficients $v_1,\ldots,v_d$ in the considered base.
Now one can easily verify by applying Lagrange multipliers that \eqref{eq:constr} under the condition $\|v\|^2=v_1^2+\ldots+v_d^2=1$ is maximized for $v$
which has coefficients $1,0,\ldots,0$ in the base $f_1,\ldots,f_d$, which
means exactly that the maximum is attained for $v=f_1$.
\end{proof}

Observe that the above result says that
the information is minimal when the projection is such that
the resulting density has the smallest possible variance (or in other
words when it is maximally concentrated).

To present intuition concerning the Cauchy-Schwarz divergence and
information potential we will consider the case of two classes
with covariances proportional to identity. The result says that the crucial in this case is the projection
onto line going through the means of both groups. Observe
that this coincides with our intuition concerning the discrimination of those
groups. Moreover, as we show in the next section, an analogous
result holds for the limiting case of arbitrary sets.

Let us recall that if the data $X$ was generated
according to the distribution $\nor(m,\Sigma)$, then 
$v^TX$ comes from the distribution $\nor(v^Tm,v^T\Sigma v)$.
Consequently, if $\Sigma=\alpha I$ and $\|v\|=1$ then 
$v^TX$ has the distribution $\nor(v^Tm,\alpha)$.

\begin{theorem} \label{th:disc}
Let $X_\plusOne,X_{\minusOne}$ be data generated by the normal densities $\nor(m_\plusOne,\alpha_\plusOne I)$ and $\nor(m_\minusOne,\alpha_\minusOne I)$ with different means $m_\plusOne \neq m_\minusOne$.
Then the maximum of 
$$
S \ni v \to D_{CS}(\nor(v^Tm_\plusOne,\alpha_\plusOne),\nor(v^Tm_\minusOne,\alpha_\minusOne))
$$ 
and simultaneously minimum of 
$$
S \ni v \to \cip(\nor(v^Tm_\plusOne,\alpha_\plusOne),\nor(v^Tm_\minusOne,\alpha_\minusOne))
$$
is attained for 
\begin{equation} \label{eq:proj}
v=\pm \frac{m_\plusOne-m_\minusOne}{\|m_\plusOne-m_\minusOne\|}
\end{equation}
\end{theorem}

\begin{proof}
Since the covariances of $X_{\plusOne}$ and $X_{\minusOne}$
equal $\alpha_\pm I$, the values of $H_2(\nor(v^T m_{\plusOne},\alpha_\plusOne))$ and $H_2(\nor(v^T m_{\minusOne},\alpha_\minusOne))$ do not depend on $v \in S$. This means the maximization of $S \ni v \to D_{CS}(\nor(v^Tm_\plusOne,\alpha_\plusOne),\nor(v^Tm_\minusOne,\alpha_\minusOne))$
is equivalent to minimization of $\cip(\nor(v^Tm_\plusOne,\alpha_\plusOne),\nor(v^Tm_\minusOne,\alpha_\minusOne))$.

Consequently, we arrive at the problem of finding minimum of
$$
\begin{array}{l}
S \ni v \to \int \nor(v^Tm_\plusOne,\alpha_\plusOne) \nor(v^Tm_\minusOne,\alpha_\minusOne) \\[0.5ex]
=\frac{1}{\sqrt{2 \pi(\alpha_\plusOne+\alpha_\minusOne)}} \exp(-\frac{1}{2(\alpha_\plusOne+\alpha_\minusOne)}\|v^T(m_\plusOne-m_\minusOne)\|^2),
\end{array}
$$
which is equivalent to the search of maximum of 
$$
S \ni v \to \|v^T(m_\plusOne-m_\minusOne)\|^2.
$$
By the Cauchy-Schwarz inequality we trivially obtain that the above
function attains its maximum for
$$
v=\pm \frac{m_\plusOne-m_\minusOne}{\|m_\plusOne-m_\minusOne\|}
$$
\end{proof}

Let us interpret the Theorem
\ref{th:disc} from the discrimination point of view. 
If we know that data from each
class comes from the normal densities $\nor(m_{\pm},\alpha_\pm I)$, then
the optimal projection from the information point of view  
is onto the line spanned by $v$ given by \eqref{eq:proj}.

\section{Theory: largest margin classifiers}

The core idea behind the Support Vector Machine model is to construct a linear classifier which
maximizes the margin between the closest samples of opposite classes. In the simplest, linearly
separable case, these closest points are refered to as support vectors. It is easy to see, that 
if we fix the value of support vector projections on $v$ to $+1/-1$ then the margin 
$\margin$ can be expressed as $1/\left \| v \right \|$ which leads to the following optimization problem.

\medskip

\noindent{\textbf{Optimization problem: Largest margin linear classifier} }

\begin{equation*}
\begin{aligned}
& \underset{v,b}{\text{maximize}}
& &  \margin = \frac{1}{\left \| v \right \|}  \\
& \text{subject to}
& & y_i( v^Tx_i  - b ) \geq 1, \; i = 1, \ldots, N
\end{aligned}
\end{equation*}

\medskip

The above problem can be reformulated in terms of 
minimal distance $\dist( v^T X_\plusOne, v^T X_\minusOne)$
between classes projections on unit length vector $v$, where
$$
\dist(P_\plusOne,P_\minusOne):=\min\{ | p_\plusOne - p_\minusOne |: 
p_\plusOne \in P_\plusOne,p_\minusOne \in P_\minusOne\}. 
$$

\medskip

\noindent{\textbf{Reformulation: Largest margin linear classifier} }

\begin{equation*}
\begin{aligned}
& \underset{v,b}{\text{maximize}}
& &  \margin = \dist( v^T X_\plusOne, v^T X_\minusOne)  \\
& \text{subject to}
& & \sign( v^Tx_i - b ) = y_i, \; i = 1, \ldots, N \\
& & & \|v\| = 1
\end{aligned}
\end{equation*}

\medskip
One of the possible generalizations of this concept lies in building a multithreshold linear classier
and maximizing all the resulting thresholds (in particular -- to maximize the smallest of the thresholds). 
Figure~\ref{fig:mar} shows that such model can increase the size of the resulting margin even in case of the 
very simple dataset consisting of four points in $\R^2$.

\begin{figure}[hbt]
\begin{center}
 \includegraphics[width=0.45\textwidth]{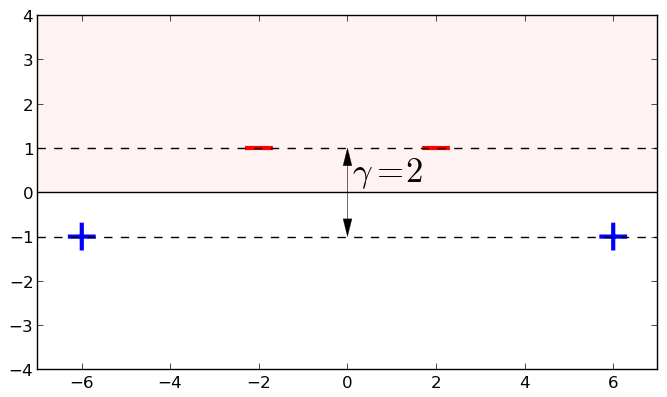}
 \includegraphics[width=0.45\textwidth]{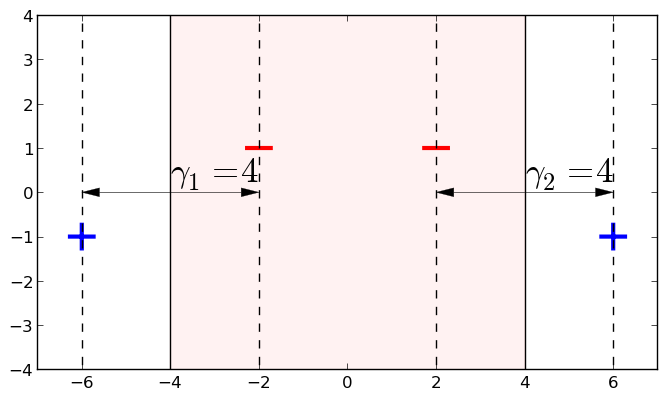}
 \caption{Example of the large margin multithreshold linear classifier (on the right) obtaining bigger margin than large margin linear classifier (on the left) on the simple dataset}
 \label{fig:mar}
\end{center}
\end{figure}
From the optimization perspective the only required modification is removal of the linear separation constraint. 

\medskip

\noindent{\textbf{Optimization problem: Largest margin multithreshold linear classifier} }

\begin{equation*}
\begin{aligned}
& \underset{v}{\text{maximize}}
& &  \margin = \dist( v^T X_\plusOne , v^T X_\minusOne )  \\
& \text{subject to}
& & \|v\| = 1   
\end{aligned}
\end{equation*}

\medskip

Such formulation can lead to arbitrary number of resulting thresholds, for example if we consider a dataset consistsing of points $x_i = i, N$, 
$y_i=(-1)^i,\; i = 1,\ldots,$.
the resulting optimal classifier would have $N-1$ thresholds of form $t_i = i + 0.5$ (for $v=1$). If we limit the number of resulting thresholds to $k$ 
(to remove the risk of overfitting) then we end up with $k-$level multithreshold linear classifier. 

\medskip

\noindent{\textbf{Optimization problem: Largest margin $k-$level multithreshold linear classifier} }

\begin{equation*}
\begin{aligned}
& \underset{v}{\text{maximize}}
& &  \margin = \dist( v^T X_\plusOne , v^T X_\minusOne )  \\
& \text{subject to}
 & & -\infty=t_0 < \ldots < t_i < \ldots < t_{k+1} = \infty \\
 & & & v^T X_\plusOne \subset \bigcup_{1 \leq i \leq k : 2|i} (t_i,t_{i+1}),\\
 & & & v^T X_\minusOne \subset \bigcup_{1 \leq i \leq k : 2|i} (t_{i-1},t_{i}), \\ 
& & & \|v\| = 1   
\end{aligned}
\end{equation*}


\medskip

It is easy to see that for $k=1$ the above problem reduces to the SVM problem which
can be solved in the polynomial time. However it appears that even in case of fixed $k=2$ the resulting decision problem is probably \mbox{NP-hard}~\cite{anthony2003learning}. In the following subsection we will introduce method of construction of such large margin
multithreshold classifier which will aim at the maximization of the margins while in the same time trying to reduce the amount of thresholds.

\subsection{Preliminaries}


A common method of density estimation is the kernel
density estimation~\cite{silverman1986density}. In the general case of data 
$P$ in 
$\R$, we typically choose a parameter $\sigma>0$ (often called
window width), and 
approximate the density of the underlying distribution by
$$
\de{P}_\sigma:=\frac{1}{|P|}\sum_{p \in P} \nor(p,\sigma^2).
$$

Although there are formulas for the optimal choice of $\sigma$ when
the data comes from the normal distribution, in general the optimal
choice of $\sigma$ is a nontrivial task which can hardly by automatized.
Intuitively, for large $\sigma$ the obtained density tends to become
one large Gaussian, while for sufficiently small we obtain almost atom 
measures at each element of $X$. In the first case we lose important information about the local properties of the data, while the second
typically leads to overfitting.

To present intuition we study the limiting cases. 
 At first we consider the limiting case
with $\sigma \to 0$, where we show that if in the linearly separable case we start from $v$ which linearly separates the data for the 
case of cross-information potential, we arrive at the 
largest margin problem (whose solution is given by SVM).
This motivates the procedure which we often apply in the next section
of starting the optimization from the SVM solution.
However, the minimization of 
cross-information potential potentially leads to
overfitting
where every point is memorized (although we still
maximize the possible margins). Next 
we formally show that for $\sigma\to \infty$ our data 
behaves (from the point of $D_{CS}$) as two large Gaussians --
we arrive at the formula which is an analogue of the results 
from the previous section. As a result we have a strong regularizing term which
prevents creation of too many thresholds\footnote{Creation of 
large number of thresholds often leads to overfitting}.
This supports the thesis that in the classification we should consider the whole
Cauchy-Schwarz divergence, as it contains the cross-information potential term, which aims to maximize the margin,
regularized by the sum of the Renyi's quadratic entropies of both classes.
This theoretical observation is supported in the next section by
empirical evaluation (see also Figure \ref{fig:ipx_dcs}), which shows that the single use of cross-information
potential in classification often leads to the unnecessary high number of thresholds,
which has consequences in suboptimal classification results.

\subsection{Margins maximization}

In this section we are going to show that minimization of the
cross information potential with small window size $\sigma \to 0$
leads to maximization of the margin width 
between the classes. Simultaneously, this will imply the
existence of  many local minima's.

We begin with the following proposition.

\begin{proposition} \label{pr:est}
Let $P_\plusOne,P_\minusOne \subset \R$ be given, $\sigma>0$. 
Then
\begin{eqnarray}
\cip(\de{P_\plusOne}_{\sigma},\de{P_\minusOne}_{\sigma}) \leq 
\frac{1}{\sqrt{2\pi}\sigma}\exp(-\tfrac{\dist^2(P_\plusOne,P_\minusOne)}{2\sigma^2}), \\
\cip(\de{P_\plusOne}_{\sigma},\de{P_\minusOne}_{\sigma}) \geq 
\frac{1}{\sqrt{2\pi}\sigma |P_\plusOne| \cdot |P_\minusOne|}
\exp(-\tfrac{\dist^2(P_\plusOne,P_\minusOne)}{2\sigma^2}).
\end{eqnarray}
\end{proposition}

\begin{proof}
Let us choose $\bar p_\plusOne \in P_\plusOne$ and $\bar p_\minusOne \in P_\minusOne$ such that
$
|\bar p_\plusOne- \bar p_\minusOne|=D=\dist(P_\plusOne,P_\minusOne)
$, then
$$
\int \tfrac{1}{|P_\plusOne|} \nor(p_\plusOne,\sigma^2)
\cdot \tfrac{1}{|P_\minusOne|} \nor(p_\minusOne,\sigma^2)
=\tfrac{1}{|P_\plusOne| \cdot |P_\minusOne|} \frac{\exp(-D^2/2\sigma^2)}{\sqrt{2\pi}\sigma}.
$$
On the other hand
$$
\int \sum_{p_\plusOne \in P_\plusOne} \tfrac{1}{|P_\plusOne|} \nor(\bar p_\plusOne,\sigma^2)
\cdot \sum_{p_\minusOne \in P_\minusOne} \tfrac{1}{|P_\minusOne|} \nor(\bar p_\minusOne,\sigma^2)
$$
$$
\leq \sum_{p_\plusOne \in P_\plusOne,p_\minusOne \in P_\minusOne}
\frac{1}{|P_\plusOne| \cdot |P_\minusOne|} \frac{\exp(-D^2/2\sigma^2)}{\sqrt{2\pi}\sigma} =\frac{\exp(-D^2/2\sigma^2)}{\sqrt{2\pi}\sigma}.
$$
\end{proof}

Given $v \neq 0$ and $\sigma>0$ we
put 
$$
\cip_\sigma(v):=\cip(\de{v^T X_\plusOne}_\sigma,\de{v^T X_\minusOne}_\sigma).
$$

We say that $v \in S$ {\em linearly separates} $X_\minusOne$ from 
$X_\plusOne$ if
$$
\inf(v^TX_\plusOne) \geq \sup(v^TX_\minusOne).
$$

First, we show that if we start the gradient descent method of
$\cip_\sigma(\cdot)$ with sufficiently small $\sigma>0$ from the vector which linearly separates
the classes, we will remain in the set of vectors which discriminate the classes.

\begin{theorem} \label{th:mar1}
We assume that $v \in S$ linearly separates $X_\minusOne$ from $X_\minusOne$ 
and that $\sigma>0$ is such that
$$
\sigma<\frac{\dist(v^TX_\plusOne,v^TX_\minusOne)}{\sqrt{2\log( |X_+| \cdot |X_-|)}}.
$$
Then steepest descent for minimization of $\cip_\sigma(\cdot)$ leads to the choice of $v'$ which also linearly separates $X_\minusOne$ from $X_\minusOne$.
\end{theorem}

\begin{proof}
Let $\V:[0,\bar t] \to S$, $\V(0)=v, \V(\bar t)=v'$ be an arbitrary continuous curve (in particular given by the seepest descent method) along which the value of $t  \to \cip_\sigma(\V(t))$ does not increase. 

Suppose that the assertion does not hold. This means that there exists
$x_+ \in X_+$, $x_- \in X_-$ and $t \in [0,\bar t]$ such that
$$
\V(t)^T x_+ \leq \V(t)^T x_-.
$$
By  the continuity we conclude that there exists $t_0 \leq t$ such that
$$
\V(t_0)^T x_+ = \V(t_0)^Tx_-.
$$
This means that $\dist(\V(t_0)^TX_+,\V(t_0)^TX_-)=0$,
and consequently by the previous proposition we get
$$
\cip_\sigma(\V(t_0)) \geq \frac{1}{\sqrt{2\pi}\sigma |X_+| \cdot |X_-|}.
$$
But from the assumptions we know that the cross-information potential
does not increase with $t$, which means that
$$
\cip_\sigma(\V(t_0)) \leq \cip_\sigma(\V(0)) \leq \frac{\exp(-\dist^2(v^TX_+, v^TX_-)/(2\sigma^2))}{\sqrt{2\pi}\sigma}.
$$
Joining this with the previous inequality, after obvious 
calculations, we obtain
$$
\dist(v^TX_+,v^TX_-) \leq \sqrt{2}\sigma \log^{1/2}( |X_+| \cdot |X_-|),
$$
a contradiction.
\end{proof}


Suppose that $X_\minusOne$ is linearly separable from $X_\plusOne$.
Let
$$
\begin{array}{l}
S_{LS}(X_\plusOne,X_\minusOne):=\\
\{v \in S: v \text{ linearly separates $X_\minusOne$ from $X_\plusOne$}\}.
\end{array}
$$
Let $\margin_{SVM}(X_\minusOne,X\plusOne)$ 
denote the maximal possible margin
along $v$ which linearly separates $X_\plusOne$ from $X_\minusOne$:
$$
\begin{array}{l}
\margin_{SVM}(X_\minusOne,X_\plusOne):= \\
\sup\{\dist(v^TX_\plusOne,v^TX_\minusOne)\,:\,
v \in S_{LS}(X_{\plusOne},X_{\minusOne})\}. 
\end{array}
$$
By $v_\svm$ we denote the unique element which realizes the above
maximum. Clearly, it is given by the normalized solution to the SVM process.

Now we are going to show that in the limiting case $\sigma \to 0$, we converge to the solution of the SVM procedure -- in other words we aim
at the maximization of the linear margin between classes.

\begin{theorem} \label{th:mar2}
Consider linearly separable classes $X_\minusOne$ and $X_\plusOne$.
Let $\bar v \in S_{LS}(X_\plusOne,X_\minusOne)$ denote an arbitrary point which realizes the minimum of the cross-information potential:
$$
\cip_\sigma(\bar v)=\min \{\cip_\sigma(v):v \in S_{LS}(X_\plusOne,X_\minusOne)\}.
$$
Then the resulting linear classifier's margin 
$\dist(\bar v^TX_+,\bar v^TX_-)$ is at least as big as
$$
\margin_{SVM}(X_\minusOne,X_\plusOne)-
\sigma \sqrt{2\log(|X_\plusOne| \cdot |X_\minusOne|)}.
$$
\end{theorem}

\begin{proof}
Let us first estimate the value of $\cip_\sigma(v_\svm)$ by applying
Proposition \ref{pr:est} with $P_{\pm}=\bar v^T X_{\pm}$
$$
\cip_\sigma(v_\svm) \leq \frac{1}{\sqrt{2\pi}\sigma}
\exp(-\tfrac{\margin(X_\minusOne,X_\plusOne)^2}{2\sigma^2}).
$$
So let us now choose an arbitrary $\bar v$ which separates $X_\plusOne$
from $X_\minusOne$ which realizes the minimum of the cross-information
potential. Then
$$
\frac{1}{\sqrt{2\pi}\sigma} \exp(-\tfrac{\margin(X_\minusOne,X_\plusOne)^2}{2\sigma^2}) \geq \cip_\sigma(v_\svm) \geq \cip_\sigma(v)
$$
$$
\geq \frac{1}{\sqrt{2\pi}\sigma |X_\plusOne|   |X_\minusOne|}
\exp(-\tfrac{\margin(v;X_\plusOne,X_\minusOne)^2}{2\sigma^2}),
$$
which directly yields the assertion of the theorem.
\end{proof}

Now we discuss the minimization of
cross-information potential over the whole $S$. It occurs that in the limiting
case $\sigma \to 0$ it results in the margin maximization.
Let $\margin(X_\minusOne,X\plusOne)$ denote the maximal possible margin along $v \in S$ 
$$
\margin(X_\minusOne,X_\plusOne):=
\sup\{\dist(v^TX_\plusOne,v^TX_\minusOne)\,:\,
v \in S\}. 
$$
By applying similar reasoning as in the proof of Theorem \ref{th:mar2} we obtain
that the minimization of cross-information potential leads
to the maximization of multiple margins in the multithreshold classifier.

\begin{theorem}
Consider classes $X_\minusOne$ and $X_\plusOne$.
Let $\bar v \in S$ denote an arbitrary point which realizes the minimum of cross-information potential:
$$
\cip_\sigma(\bar v)=\min \{\cip_\sigma(v):v \in S\}.
$$
Then the resulting multithreshold linear classifier's margins 
$\dist(\bar v^TX_+,\bar v^TX_-)$ are at least as big as 
$$
\margin(X_\minusOne,X_\plusOne)-
\sigma \sqrt{2\log(|X_\plusOne| \cdot |X_\minusOne|)}.
$$
\end{theorem}

The above theorem leads
to the conclusion that cross-information potential (without
additional regularizing terms) leads to the construction of largest margin
multithreshold classifier. However, it lacks the ability to control 
the number of resulting thresholds and as a result, for sufficiently
small $\sigma$, it may construct an interval for each data point, which
leads to overfitting. A real life example of sonar dataset from UCI
repository is given in Figure~\ref{fig:ipx_dcs}, which shows the
comparison of minimization of $\cip$ versus maximization of
Cauchy-Schwarz divergence. \begin{figure}[hbt]
\begin{center}
 \includegraphics[width=0.45\textwidth]{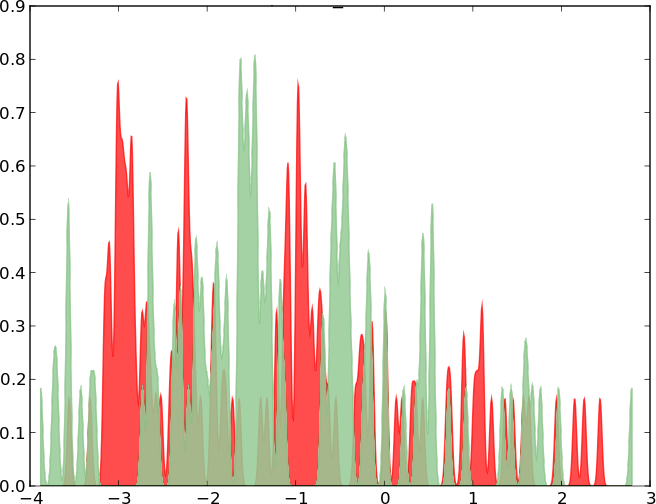}
 \includegraphics[width=0.45\textwidth]{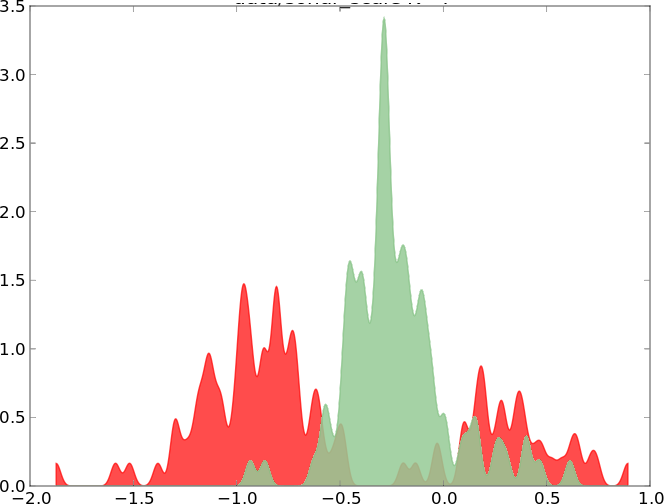}
\end{center}
 \caption{Sample kernel density estimation of projected sonar dataset with small $\sigma$ using $\cip$ optimization (on the left) and $D_{CS}$ (on the right).}
 \label{fig:ipx_dcs}
\end{figure}Following section shows how introduction
of the classes' entropies to the optimization process causes reduction
of the model's complexity (in a limiting case to the linear classifier). 

\subsection{Regularization}

We show that the analogue of the Theorem \ref{th:disc} holds also for the limiting case when we increase the window width to infinity. This will result in construction of the linear classifier (limiting reduction of the number of thresholds).

We consider only the case when the
window width $\sigma$ is set to be equal for both classes.
We recall that for $P \subset \R$ and $\sigma>0$ we put 
$$
\de{P}_\sigma:=\frac{1}{|P|}\sum_{p \in P} \nor(p,\sigma^2).
$$
Thus $\de{P}_{\sigma}$ denotes the kernel density estimation based on the
set $P$ with window width $\sigma$.
We begin with the following observation.

\begin{proposition} \label{pr:mean}
Let $P_\plusOne,P_\minusOne \subset \R$ be given, and let the window width $\sigma>0$ be fixed. Then
\begin{equation} \label{eq:mean}
D_{CS}(\de{P_\plusOne}_{\sigma},\de{P_\minusOne}_{\sigma})= \frac{1}{2\sigma^2}(m_\plusOne-m_\minusOne)^2
+\mathcal{O}(\sigma^{-4}) \text{ as } \sigma \to \infty,
\end{equation}
where $m_{\pm}$ denote the means of $P_{\pm}$.
\end{proposition}

\begin{proof}
We denote elements of $P_\plusOne$ by $p_\plusOne$,
and elements of $P_\minusOne$ by $p_\minusOne$.

We have
$$
\begin{array}{l}
\int \de{{P_\plusOne}}_\sigma \de{{P_\minusOne}}_\sigma=\frac{1}{|{P_\plusOne}||{P_\minusOne}|}\sum_{p_\plusOne,p_\minusOne} \nor(p_\plusOne-p_\minusOne,2\sigma^2)\\[0.5ex]
=
\frac{1}{\sqrt{4\pi \sigma^2}|{P_\plusOne}||{P_\minusOne}|}\sum_{p_\plusOne,p_\minusOne} \exp(-(p_\plusOne-p_\minusOne)^2/(4\sigma^2)).
\end{array}
$$
Since $\exp(h)= 1+h+\mathcal{O}(h^2)$ and $\log(1+h)=1+h+\mathcal{O}(h^2)$ for small $h$,  the above equality implies that for large $\sigma$
$$
\begin{array}{l}
\log \int \de{{P_\plusOne}}_\sigma \de{{P_\minusOne}}_\sigma \\[0.5ex]
= -\log(2\sigma\sqrt{\pi})\\
\phantom{=}+\log(1-\frac{1}{4|{P_\plusOne}||{P_\minusOne}|\sigma^2}\sum_{p_\plusOne,p_\minusOne} (p_\plusOne-p_\minusOne)^2)+O(\sigma^{-4}) \\[0.5ex]
= -\log(2\sigma\sqrt{\pi})-\frac{1}{4|{P_\plusOne}||{P_\minusOne}|\sigma^2}\sum_{p_\plusOne,p_\minusOne} (p_\plusOne-p_\minusOne)^2
+\mathcal{O}(\sigma^{-4}).
\end{array}
$$
Consequently
$$
\begin{array}{l}
D_{CS}(\de{{P_\plusOne}}_\sigma,\de{{P_\minusOne}}_\sigma) \\[0.5ex]
=\log \int \de{{P_\plusOne}}_\sigma^2+\log \int \de{{P_\minusOne}}_\sigma^2-2\log \int \de{{P_\plusOne}}_\sigma \de{{P_\minusOne}}_\sigma
\\[0.5ex]
=\frac{1}{4\sigma^2}(-\tfrac{1}{|{P_\plusOne}|^2}\sum_{p_\plusOne,p'_\minusOne} (p_\plusOne-p'_{\plusOne})^2-\tfrac{1}{|{P_\minusOne}|^2}\sum_{p_\minusOne,p'_\minusOne} (p_\minusOne-p'_{\minusOne})^2\\
\phantom{=}+\tfrac{2}{|{P_\plusOne}||{P_\minusOne}|}\sum_{p_\plusOne,p_\minusOne} (p_\plusOne-p_\minusOne)^2)+\mathcal{O}(\sigma^{-4}).
\end{array}
$$
By applying obvious calculations we obtain that
$$
\begin{array}{l}
-\frac{1}{|{P_\plusOne}|^2}\sum_{p_\plusOne,p'_\plusOne} (p_\plusOne-p'_{\plusOne})^2-\frac{1}{|{P_\minusOne}|^2}\sum_{p_\minusOne,p'_\minusOne} (p_\minusOne-p'_{\minusOne})^2 \\
\phantom{\frac{1}{4\sigma^2}[}+\frac{2}{|{P_\plusOne}||{P_\minusOne}|}\sum_{p_\plusOne,p_\minusOne} (p_\plusOne-p_\minusOne)^2
\\[0.5ex]
=-2 (\frac{1}{|{P_\plusOne}|}\sum_{p_\plusOne} p_\plusOne^2-
(\frac{1}{|{P_\plusOne}|}\sum_{p_\plusOne} p_\plusOne)^2)
\\
\phantom{=}-2(\frac{1}{|{P_\minusOne}|}\sum_{p_\minusOne} p_\minusOne^2-
(\frac{1}{|{P_\minusOne}|}\sum_{p_\minusOne} p_\minusOne)^2) \\
\phantom{=}+\frac{2}{|{P_\plusOne}|}\sum_{p_\plusOne} p_\plusOne^2+\frac{2}{|{P_\minusOne}|}\sum_{p_\minusOne} p_\minusOne^2
-\frac{4}{|{P_\plusOne}||{P_\minusOne}|}\sum_{p_\plusOne} p_\plusOne \sum_{p_\minusOne} p_\minusOne
\\[0.5ex]
=-4m_{\plusOne}m_{\minusOne}+2m_{\plusOne}^2+
2m_{\minusOne}^2=2(m_{\plusOne}-m_{\minusOne})^2.
\end{array}
$$
\end{proof}

Observe that the constant $C$ in $\mathcal{O}(\sigma^{-4})=C\sigma^{-4}$  in the \eqref{eq:mean} can be
estimated from the proof by an increasing function of 
$\sum_{p_\plusOne} | p_\plusOne |^2+\sum_{p_\minusOne} |p_\minusOne|^2$.

\begin{theorem}
We consider classes $X_\plusOne$ and $X_\minusOne$. We assume additionally that class centers $m_\pm \in \R^d$ are such that $m_\plusOne  \neq m_\minusOne$, where $m_\pm$ denote the means of $X_\pm$. We put 
$$
v_\infty=\frac{m_\plusOne-m_\minusOne}{\|m_\plusOne-m_\minusOne\|}.
$$

For $\sigma>0$ let $v_\sigma \in S$ denote the argument for which the function
$$
D_{CS}^\sigma(v):S \ni v \to D_{CS}(\de{v^TX_{\plusOne}}_\sigma,\de{v^TX_{\minusOne}}_\sigma)
$$
takes the maximum value.
Then $v_\sigma$ tends to $\pm v_\infty$ with increasing $\sigma$, that is
$$
\min(\|v_{\sigma}-v_\infty\|,\|v_\sigma+v_\infty\|)=\mathcal{O}(\sigma^{-1}) 
\text{ as } \sigma \to \infty.
$$
\end{theorem}

\begin{proof}
Clearly, by the Proposition \ref{pr:mean}
\begin{equation} \label{eq:osz}
\frac{2\sigma^2}{\|m_\plusOne-m_\minusOne\|^2}D_{CS}^\sigma(v)=
\il{v,v_{\infty}}^2+\mathcal{O}(\sigma^{-2}),
\end{equation}
and the constant in $\mathcal{O}(\sigma^{-2})$ can be bounded by
an increasing function of $\sum_{x_\plusOne} \|x_\plusOne\|^2+\sum_{x_\minusOne} \|x_\minusOne\|^2$.

Consider $v_\sigma \in S$. Without loss of generality
by taking $-v_\sigma$ in place of $v_\sigma$, if necessary, and applying the
fact that $D^\sigma_{CS}$ is an even function, we may assume
that $v_\sigma$ is nearer to $v_{\infty}$ than to $-v_\infty$, that is 
$
\|v_\sigma-v_{\infty}\| \leq \|v_\sigma+v_{\infty}\|. 
$
We are going to estimate from above the value of
$
\|v_\sigma-v_\infty\|.
$
Observe first that
\begin{equation} \label{eq:dista}
\begin{array}{c}
\il{v_\sigma,v_\infty}=\frac{1}{2}(\|v_\sigma\|^2+\|v_\infty\|^2-\|v_\sigma-v_\infty\|^2) \\[0.3ex]
=1-\tfrac{1}{2}\|v_\sigma-v_\infty\|^2.
\end{array}
\end{equation}
This trivially yields that $\il{v_\sigma,v_\infty} \geq 0$.

On the other hand, since $D_{CS}^\sigma$ takes maximum in $v_\sigma$, 
by applying \eqref{eq:osz} twice, we get
$$
\begin{array}{l}
\il{v_\sigma,v_\infty}^2 
\geq \frac{2\sigma^2}{\|m_\plusOne-m_\minusOne\|^2}D_{CS}^\sigma(v_\sigma)
-C'\sigma^2 \\[0.5ex]
\geq \frac{2\sigma^2}{\|m_\plusOne-m_\minusOne\|^2}D_{CS}^\sigma(v_\infty)
-C'\sigma^2
\\[0.5ex] \geq \il{v_\infty,v_\infty}^2-C''\sigma^{-2}= 1-C''\sigma^{-2}
\end{array}
$$
for certain $C',C''>0$. 
Since $\sqrt{1-h} \geq 1-h$ (for $h \geq 0$) and $\il{v_\sigma,v_\infty}$ is nonnegative, this yields that
$$
\il{v_\sigma,v_\infty} \geq \sqrt{1-C''\sigma^{-2}} \geq 1-C''\sigma^{-2}.
$$
By applying \eqref{eq:dista} we conclude that
$\|v_\sigma-v_\infty\|^2< 2C''\sigma^{-2}$, which yields
$$
\|v_\sigma-v_\infty\| =\mathcal{O}(\sigma^{-1}).
$$
\end{proof}

%

\subsection{Classification theory}

Let us recall that the objective function 
$$
D_{CS}(f,g)=-(H_2(f)+H_2(g)+2\log \cip(f,g)),
$$ 
consists of two parts, the entropy term $H_2(f)+H_2(g)$ which serves the regularization purpose and $\cip(f,g)$ which ensures optimal discrimination of
the classes. Maximization of $D_{CS}$ and classifying data based on the 1-dimensional kernel density estimation  leads to the construction of multithreshold linear classifier. Optimization procedure tries to simultaneously maximize the margins between classes and to minimize the number of resulting thresholds.

As Anthony~\cite{anthony2004generalization} showed, the considered class of classifiers have bounded generalization error dependent on the number of thresholds $k$:

\begin{theorem}{Generalization bounds (Anthony, 2004~\cite{anthony2004generalization})}
 With probability at least $1-\delta$, for $N$ points in $\mathbb{R}^d$: 
 $$
 \mathrm{E} \leq \mathrm{E}_{emp} + \sqrt{ \tfrac{8}{N}\left ( (d+k-1)\log \left ( \tfrac{2eNk}{d+k-1} \right ) + \log \left ( \tfrac{14k^2}{\delta} \right ) \right ) }
 $$
 where $\mathrm{E}$ is the generalization error and $\mathrm{E}_{emp}$ is the training (empirical) error.
\end{theorem}

As it has been previously shown, minimization of the Renyi's entropy leads to the choice of projections where each class is as condensed as possible. In a natural way this means that this process leads to the minimization of number of resulting thresholds (the value of estimated density is monotonically decreasing when we move away from the closest point with Gaussian function centered in it). 

The following theorem shows that for $k-$level threshold linear classifier restricted to the sphere, the generalization bounds can be improved by maximizing the margin $\margin$.

\begin{theorem}{Generalization bounds with margin (Anthony, 2004~\cite{anthony2004generalization})}
 With probability at least $1-\delta$, for $N$ points in $\mathbb{R}^d$ such that $\|x_i\|\leq 1$, $\|v\|=1$ and margin $\margin \in (0,1]$:
 
 $$
 \mathrm{E} \leq \mathrm{E}_{emp} + \sqrt{ \tfrac{8}{N}\left ( \tfrac{1152}{\margin^2} \log \left ( 9N \right ) + k \log \left ( \tfrac{10}{\margin} \right ) + \log \left ( \tfrac{4}{\delta} \right ) \right ) }
 $$
\end{theorem}

According to Theorems \ref{th:mar1} and \ref{th:mar2} minimization of $\cip$ leads (in the limiting case) to the maximization of the separating margins. So our method is truly aimed at structural risk minimization. We search for such multithreshold linear classifier which minimizes the generalization error through selection of the structurally simplest hypothesis. This shows another similarity to the SVM model, but adapted to multithreshold case.

\section{Practical considerations}

In this section we deal with some practical considerations regarding our optimization problem, which lies in maximizing the Cauchy-Schwarz divergence of the kernel density estimation of projections of our data. As it
has been proven in Theorem~\ref{th:scale}, this problem is scale invariant, so we can constrain domain
of searched parameters into the unit sphere in $\R^d$. In practice this limitation reduces not only the parameter space, but also the risk of numerical instability, while coming at no additional computational cost. 

In the optimization we apply the typical steepest ascent approach. It is a common knowledge that such procedure can be performed for maximization of function $f$ on the sphere by simply projecting the gradient onto the tangent hyperplane and performing the usual line search procedure on the big circle given by gradient's direction. Given the starting point $v_0$ such that $\| v_0 \|=1$ it can be expressed as following iterative procedure for step sizes $\alpha_t$:
\begin{equation*}
 \begin{aligned}
   h_t = &  \nabla_v f(v_t) - \langle \nabla_v f(v_t) , v_t \rangle v_t, \\    
   v_{t+1} = &  v_t \cos( \alpha_t ) +  \sin( \alpha_t ) h_t/\|h_t\|.
 \end{aligned}
 \end{equation*}
Let us now summarize the problem of $D_{CS}$ maximization with Silverman's rule for kernel density estimator window width.

\medskip



\noindent{\bf \RMC{} optimization problem. }{\em 
Given sets $X_\plusOne, X_\minusOne \subset \R^d$}
\begin{equation*}
\begin{aligned}
& \underset{v \in \mathbb{R}^d}{\text{maximize }}
-( H_2^{{\plusOne}}(v) + H_2^{{\minusOne}}(v) + 2 \log \cip_{X_{\plusOne}X_{\minusOne}}(v) ) \\
& \text{subject to }
 \left \| v \right \| = 1 \\
& \text{where} \\ 
&  
H_2^{\pm}(v)=-\log \cip_{X_\pm X_\pm}(v) \\
& 
\cip_{AB}(v)= \tfrac{1}{\sqrt{2\pi V_{AB}(v)}\cdot |A||B|} 
\sum_{a \in A,b  \in B}
\exp\big(-\tfrac{\langle v,a -b\rangle^2}{2V_{AB}(v)}\big),  \\
&  V_{AB}(v) = V_A(v)+V_B(v), \\
&  V_{A}(v) = ( (4/3)^{1/5}|A|^{-1/5} \sigma_{v^T A} )^2.
\end{aligned}
\end{equation*}

In order to perform steepest ascent optimization we need to compute gradient of $D_{CS}$ function.
We present its final formula and omit its obvious derivation.

\begin{equation*}
\begin{aligned}
\nabla D_{CS}(v)= & \frac{\nabla \cip_{X_\plusOne X_\plusOne}(v)}{\cip_{X_\plusOne X_\plusOne}(v)}+
\frac{\nabla \cip_{X_\minusOne X_\minusOne}(v)}{\cip_{X_\minusOne X_\minusOne}(v)}-
\frac{2\nabla \cip_{X_\plusOne X_\minusOne}(v)}{\cip_{X_\plusOne X_\minusOne}(v)},\\
\nabla \cip_{AB}(v) = & \tfrac{1}{2 V_{AB}(v)\sqrt{2\pi V_{AB}(v)}\cdot |A||B|} \sum_{a \in A,b \in B} \exp (-\tfrac{\langle v,a-b \rangle^2}{2V_{AB}(v)} )\\
& \left\{
 (\tfrac{\langle v,a-b \rangle^2}{2V_{AB}(v)}-1 )\nabla V_{AB}(v)-2\langle v,a-b \rangle (a-b)
\right\},\\
\nabla V_{AB}(v) =& \nabla V_{A}(v) + \nabla V_{B}(v),\\
 \nabla V_{A}(v) =& \frac{\left ( \frac{4}{3} \right )^\frac{2}{5}}{|A|^{12/5}}
\left(
|A| \cdot \sum_{a \in A}  \langle v,a \rangle a-\sum_{a \in A} \langle v,a \rangle \cdot \sum_{a \in A}a
\right).
\end{aligned}
\end{equation*}

It is easy to see that computation of both function value and its gradient is computationally expensive ($\mathcal{O}(N^2)$, where $N$ is the size of the training set). This issue is partially compensated by fact that in practice it is sufficient to perform just few steps of this process in order to find a local maxima. In the basic approach we choose random starting points from the sphere, run optimizations from them and select the one yielding the biggest value. However, it is also possible to start optimization from the solution given by some computationally cheap model, like for example a perceptron or a linear SVM with $C=1$. As a result, we can obtain a reasonable solution in quite short time (using just one optimization procedure). These methods are further investigated in the Evaluation Section.


\subsection*{Classifier complexity}
Classification using the actual density estimators on $\R$ requires $\mathcal{O}(N)$ operations (each training point has impact on the classification). This issue can be overcomed by constructing the actual $k$-threshold linear classifier from this density by search for points $t_1,\dots,t_k$ such that $\de{v^TX_\plusOne}(t_i)=\de{v^TX_\minusOne}(t_i)$ (see algorithm in Figure~\ref{al:construction}). 
\begin{figure}[H]
\begin{algorithmic}[1]
\STATE $x_1,...,x_N \gets \textbf{sort}(v^TX_\plusOne \cup v^TX_\minusOne$)
\STATE $Q \gets \de{v^TX_\plusOne}-\de{v^TX_\minusOne}$
\STATE $k\gets0$
\FOR { $i=2$ \textbf{to} $N$}
\IF {$\sign( Q(x_{i-1}) ) \neq \sign( Q(x_i) )$}
    \STATE $k\gets k + 1$
    \STATE $t_k \gets \textbf{binsearch}_{ x \in (x_{i-1},x_i) } Q(x)=0$
\ENDIF
\ENDFOR
\RETURN       $t_1,...,t_k $            

\end{algorithmic}
\caption{$k$-threshold linear classifier construciton for kernel density estimation of $X_\pm$ projections on given $v$} 
\label{al:construction}
\end{figure}
As a result, classification's complexity of the new points is decreased to $\mathcal{O}(d+\log(k))$ (binary search of $k$ midpoints on $\R$). In the Evaluation Section we also show that it is sufficient to run just few iterations of \textit{binsearch} to build such classifier. However, such operation destroys easy access to the estimation of $P(y|x)$, as similarly to other linear models we just have thresholds. In order to obtain such probabilities (confidences) one still needs to query all the training points. If such approach is too expensive one can change used kernel to the Epanechnikov or other with finite support.

\subsection*{Parameterization}
One can use different kernel width estimator by alternating the $V_{A}(v)$ term (and its gradient). In particular, in order to include the kernel window width scaling factor $\gamma$ it is sufficient to replace the variance term $V_A(v)$ in the previous equations with
$V_A^\gamma(v) := \gamma^2 V_A(v)$, and analogously $\nabla V_{A}(v)$ becomes $\nabla V_{A}^\gamma(v) := \gamma^2 \nabla V_{A}(v)$. As shown in the Evaluation section, this can be beneficial as the Silverman's rule tends to overestimate the required value~\cite{cao1994comparative}. Size of the $\gamma$ factor plays also the role of a bias--variance tradeoff coefficient. With bigger values the optimization will lead to the very simple 
single-threshold models (with a limiting case proven in the Theory Section), while the small values can lead to overfitting the data. Default value of $\gamma=1$ yields quite reasonable solutions (as showed in the Evaluation Section) but results can be improved by searching for (in most cases) smaller values. In general, regularization strength grows with $\gamma$.

It is also possible to include samples weights $w_x$ directly in the proposed formulation. The only modification needed is to put the weighted kernel density estimator 
$$ 
 \de{P}_\sigma = \frac{1}{\sum_{p \in P} w_p} \sum_{p \in P} w_p\nor(p,\sigma^2).
$$
\noindent It is worth noting that this weighting works on the basis of in-class weights, it cannot be directly applied to weight the whole class. On the other hand similar concept can be used to include the known input data uncertainty measure by using different $\sigma_x$ for each point.

\subsection*{Non-linear case}

For problems requiring non-linear model one can adapt the proposed approach. Direct kernelization would lead to higher computational complexity ($\mathcal{O}(N^3)$ per iteration), but there are other possible solutions. First, one can use Nystrom's method of kernel approximation~\cite{drineas2005nystrom} which does not require such complex operations. It is also possible to apply random projection techniques~\cite{huang2006extreme, huang2011extreme, hegde2007random}, which map the input space through some non-linear function (eg. RBF) as the preprocessing step. In particular, one can use clustering methods to seed the position of RBF function (as it is done in RBF networks~\cite{haykin2009neural}). This problem is, however, beyond the scope of this work and should be the topic of future research.

\section{Evaluation}

%
%

We evaluated our method using code written in \texttt{C++} with help of \texttt{boost}~\cite{karlsson2005beyond} library. Experiments were coducted on an Intel Xeon 2.67GHz machine. In the first phase we used ten well known UCI~\cite{asuncion2007uci} binary datasets, briefly summarized in Table~\ref{tab:summary_datasets}. 
\begin{table}[H]
\begin{center}
\begin{tabular}{lrrrr}
\hline
dataset	&	d 	& n & $|X_{\plusOne}|$ & $|X_{\minusOne}|$ \\
\hline
australian 	& 14 & 690& 383 & 307 \\
breast cancer 	& 9 & 683 & 444 & 239 \\
diabetes 	& 8 & 768 & 268 & 500  \\
fourclass 	& 2 & 862 & 307 & 555 \\
german number 	& 24 & 1000 & 700 & 300  \\
heart 		& 13 & 270 & 150 & 120  \\
ionosphere 	& 34 & 351 & 225 & 126 \\
liver-disorders	& 6 & 345 & 145 & 200  \\
sonar 		& 60 & 208 & 111 & 97  \\
splice 		& 60 & 1000& 483 & 517 \\
\hline
\end{tabular} 
\end{center}
\caption{Summary of UCi datasets used in tests.}
\label{tab:summary_datasets}
\end{table} 
During the second part of the evaluation we focused on real, cheminformatics data regarding chemical compounds activity prediction for selected proteins. For different models implementations we used the \texttt{scikit-learn}~\cite{pedregosa2011scikit} package, implementing the popular \texttt{libSVM}~\cite{chang2011libsvm} library. Three evaluation metrices are used in further parts of our paper: accuracy (ACC), Matthew's Correlation Coefficient (MCC) and weighted accuracy (WAC, also known as averaged accuracy).

\subsection{Toy dataset}

For better understanding of our method's characteristic we begin evaluation with XOR like dataset, composed of 100 samples from four two-dimensional Gaussians centered in points $(-1.-1), (+1,+1)$ (positive samples) and $(-1,+1), (+1,-1)$ (negative ones), see Figure~\ref{fig:toy}.
\begin{figure}[H]
 \begin{center}
  \includegraphics[width=3.5cm]{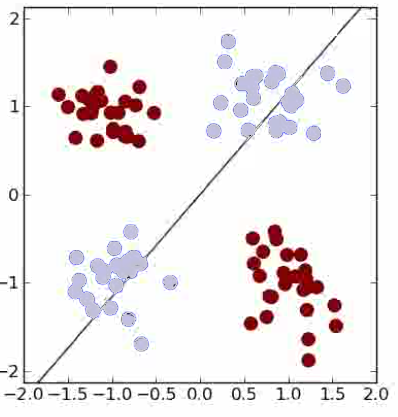}
  \includegraphics[width=3.5cm]{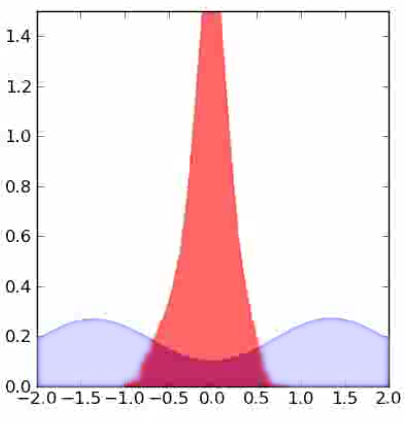}
  \includegraphics[width=3.5cm]{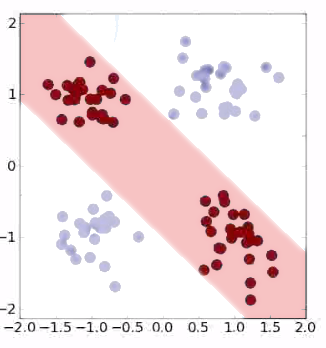}
  \\
  \includegraphics[width=3.5cm]{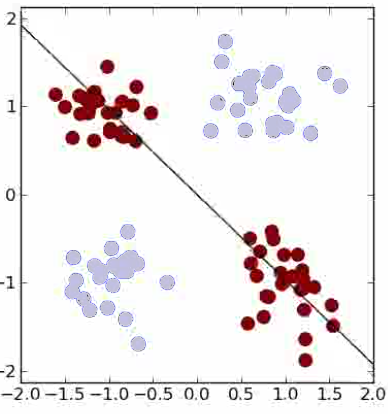}
  \includegraphics[width=3.5cm]{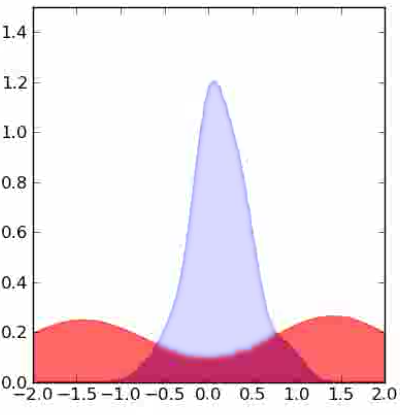}
  \includegraphics[width=3.5cm]{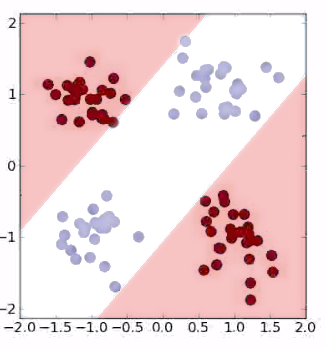}
  \caption{XOR like dataset composed of four Gaussians}
  \label{fig:toy}
 \end{center}
 
\end{figure}
Obviously this dataset is not linearly separable, but can be shattered with use of a 2-threshold linear classifier. In terms of $D_{CS}$ this data has two (up to the center symmetry) local maxima, one around $v_1=(\sqrt{2}/2,\sqrt{2}/2)$ and one around $v_2=(-\sqrt{2}/2,\sqrt{2}/2)$.
Solution given by $v_1$ has higher $D_{CS}$ as the spiked class is much narrower (its Renyi's entropy is lower), and as a result -- smallest of the two resulting margins is bigger. One can notice, that density estimation with Silverman's rule tends to overestimate the required kernel window size (splitted class is too flat).

Proposed method achieved almost 100\% scores under all considered metrices, while the linear models (both perceptron and SVM) achieved at most 50\% accuracy. Naturally, if kernelized with polynomial kernels, these methods would perform much better. This is however only a simple example to illustrate the potential benefits of multithreshold classifier while still using only the linear projection.

\subsection{Impact of regularization}

In the Theoretical Section we showed that minimization of $\cip$ leads to the separation with the large margin. However, if the chosen kernel width is too small, this may lead to overfitting issues due to the multithreshold nature of our model. In the worst case scenario, when our density estimation degenerates to almost atomic measure, we would get a perfect training set fitting with no generalization capabilities. This supports the need for the regularization based on the each classes densities' entropies and as a result optimization of $D_{CS}$ instead of just $\cip$. On the Figure~\ref{fig:K} one can find histogram of number of thresholds in our model for UCI datasets. We use Silverman's rule for kernel width estimation, which is known to rather overestimate this value (the optimal kernel width is often smaller than the one given by Silverman). However, even in such case one can notice, that purely $\cip$ based optimization leads to significantly more complex models (with higher number of 
thresholds)
. 
\begin{figure}[H]
\begin{center}
 \includegraphics[width=0.6\textwidth]{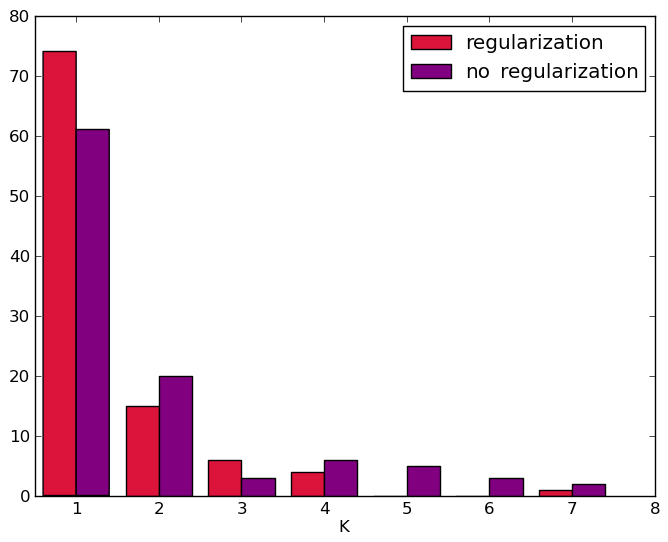}
\end{center}
 \caption{Histogram of number of resulting thresholds in classifiers built on the UCI datasets}
 \label{fig:K}
\end{figure}

\subsection{$D_{CS}$ and generalization}

In the previous sections we argued that maximialization of the Cauchy-Schwarz divergence should lead to the choice of a model with good generalization capabilities. In the Figure~\ref{fig:dcs_gen_cor} one can see how value of $D_{CS}$ is correlated with the Matthew's Correlation Coefficient (measured on the test set) for splice dataset. 
\begin{figure}[H]
\begin{center}
 \includegraphics[width=0.6\textwidth]{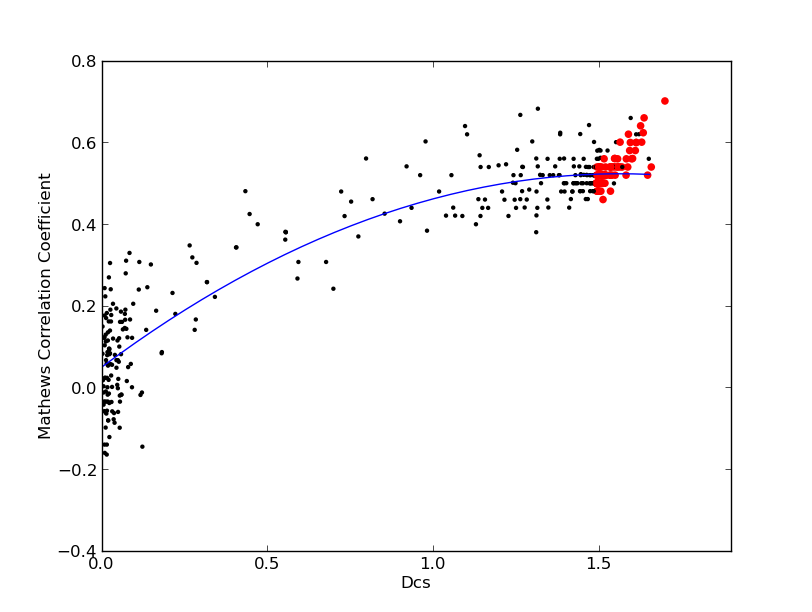}
\end{center}
 \caption{Correlation between $D_{CS}$ value (on x axis) and generalization capabilities (expressed as MCC on the test sets in 10 CV) for the splice dataset.h Big dots represent local maxima of $D_{CS}$ found during opitmization process.}
 \label{fig:dcs_gen_cor}
\end{figure}
Easily noticable relation suggests that $D_{CS}$ can be truly used as a criterion for the choice of model. Pearson's correlation coefficient between these two values for splice is about $0.9$. It seems also, that it is rather resistant to the overfitting (as there is no noticable decrease in the generalization for high $D_{CS}$ values). Correlations for the remaining datasets are reported in Table~\ref{tab:cor_gen}, all of them are statistically significant (in terms of correlation p-value). 
\begin{table}[H]
\begin{center}
\begin{tabular}{lrrr}
\hline
dataset	&	ACC & MCC & WAC \\
\hline
australian 	& 0.898 & 0.900 & 0.902 \\
breast cancer 	& 0.901 & 0.897 & 0.896 \\
diabetes 	& 0.494 & 0.611 & 0.624 \\
fourclass 	& 0.245 & 0.374 & 0.393 \\
german number 	& 0.407 & 0.569 & 0.575 \\
heart 		& 0.726 & 0.726 & 0.728 \\
ionosphere 	& 0.537 & 0.532 & 0.518 \\
liver-disorders	& 0.348 & 0.350 & 0.357 \\
sonar 		& 0.645 & 0.635 & 0.644 \\
splice 		& 0.943 & 0.941 & 0.943 \\
\hline
\end{tabular} 
\end{center}
\caption{Mean correlation between $D_{CS}$ and the generaliztaion capabilities across 10-folds of cross validation.}
\label{tab:cor_gen}
\end{table}
First, we see that in all cases there is a moderate to strong positive correlation. Second, these results confirm that our method is aimed at balanced measures (like WAC and MCC) while in the same time not well suited for accuracy (which by its definition prefers non-balanced models). In further part of our paper we focus only on these two metrics.

%
\subsection{UCI binary classification}

In the following part we will compare the efficiency of Multithreshold Entropy Linear Classifier (\RMC), Support Vector Machines (SVM), Support Vector Machines with class balancing (SVM-B) and Perceptron. SVM-B is the SVM model with $C$ value splitted into $C_\plusOne$ and $C_\minusOne$ invertibly proportional to the corresponding class sizes. All experiments are performed in 10-fold cross validation.

We first investigated how well these four models work when ran with default parameters (as given in \texttt{scikit-learn} library, which means $C=1$ for SVM models). Figure~\ref{fig:UCI_no_tuning} shows obtained results in terms of WAC measure (results for MCC were analogous). 
\begin{figure}[h]
 \begin{center}
 \includegraphics[width=\textwidth]{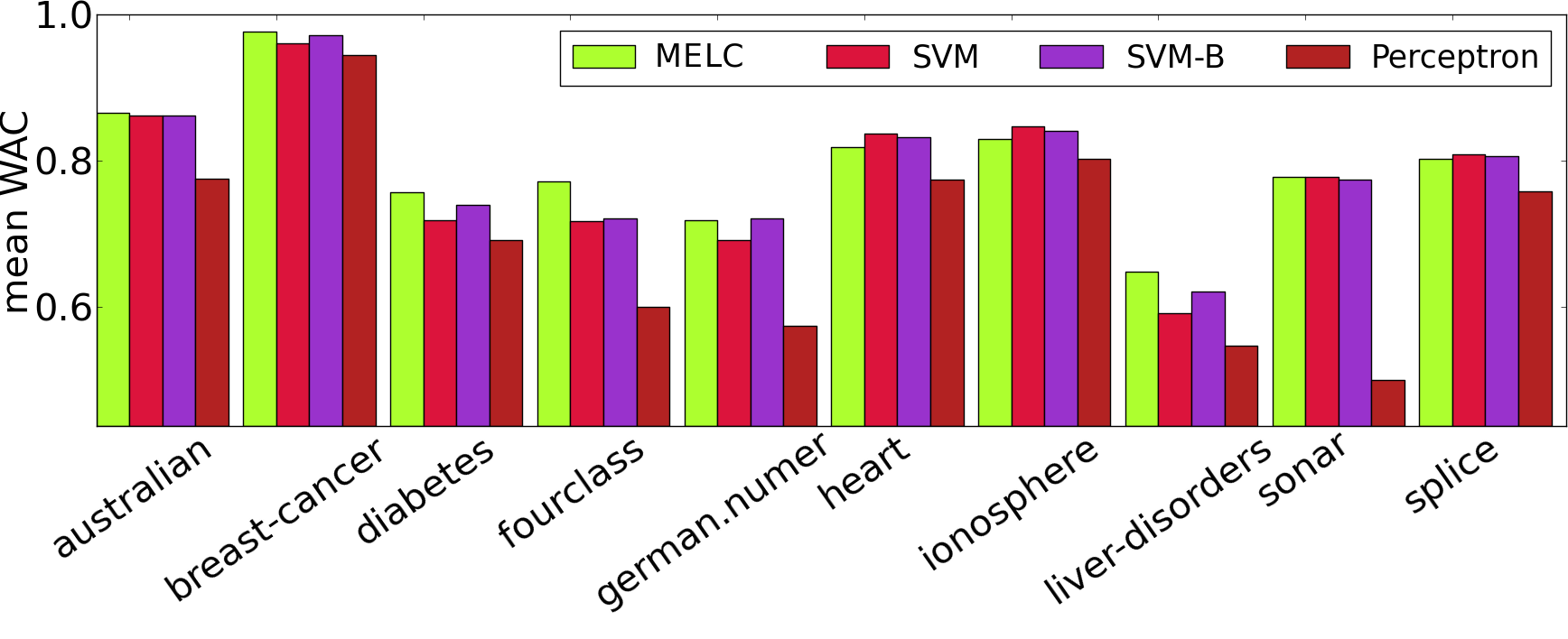}
 \end{center}
 \caption{Comparision of 10-fold cross validation WAC scores with default parameters}
 \label{fig:UCI_no_tuning}
\end{figure}
Without tuning of any model, \RMC{} obtained results comparable with SVM for most datasets, and outperformed it for a few (including liver-disorders, fourclass and diabetes). Results of perceptron were significantly worse in all cases. In nine of ten datasets \RMC{} build a linear classifier, and in case of fourclass dataset, a 3-threshold linear classifier. This supports our claim that the regularization prevents model from selecting too big $k$ values. However, it is worth noting that even though \RMC{} gained similar mean WAC as SVMs for some problems, it built different decision models. In particular, after investigation of results of individual folds, in some cases our method significantly outperformed SVM and vice-versa. It supports our claim, that even though there are important theoretical connections between these models, they result in different classifiers.

As it was previously stated, process of optimization of $D_{CS}$ may be computationally expensive. To deal with this problem one can perform single (or few) gradient based optimization from solutions given by some other, cheapier models. Comparison of the results obtained by our approach seeded with $v$ found by SVMs and perceptron are plotted in Figure~\ref{fig:UCI_init}. 
\begin{figure}[h]
 \begin{center}
 \includegraphics[width=\textwidth]{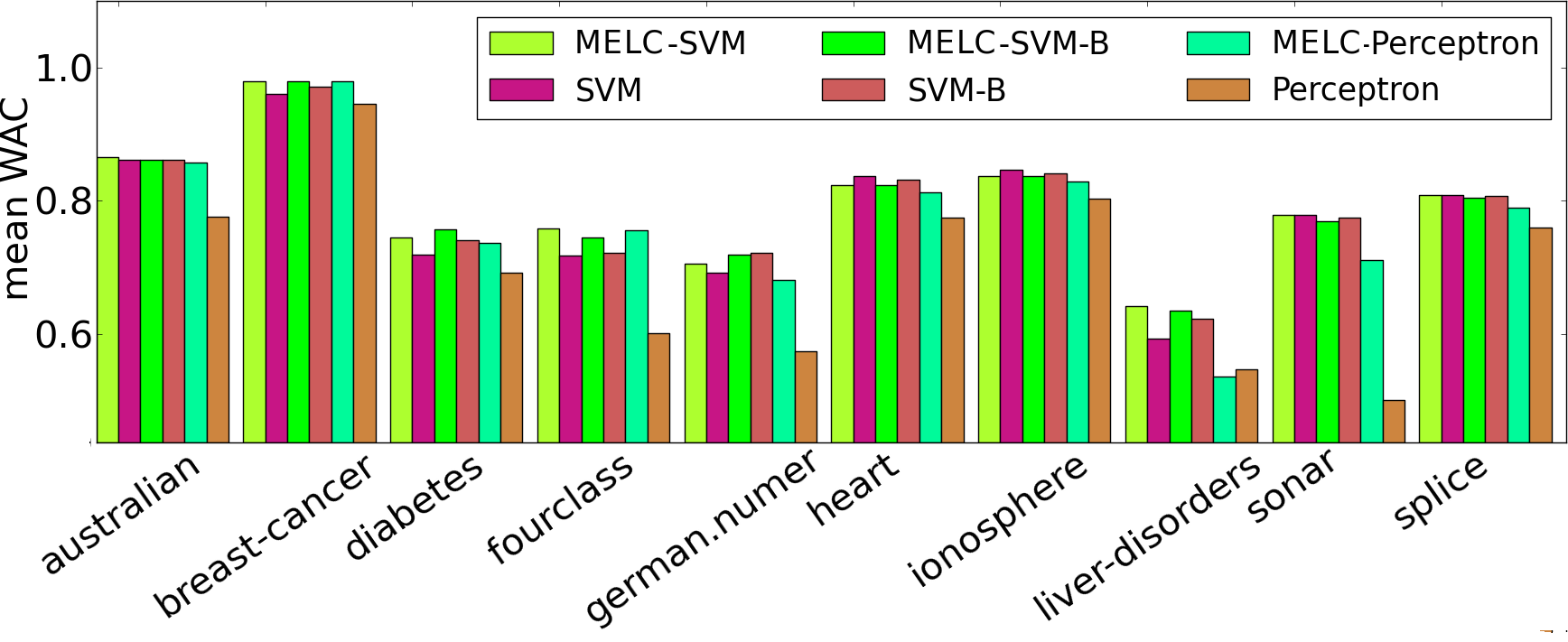}
 \end{center}
 \caption{Comparision of 10-fold cross validation WAC scores for \RMC{} starting from solution given by SVM, SVM-B and perceptron (with default parameters)}
 \label{fig:UCI_init}
\end{figure}
One can notice, that such initialization can lead to quite reasonable solutions. Starting from perceptron solutions generally lead to much worse scores, as this model finds completely different type of solutions than \RMC{} does. In case of SVM it seems possible to exploit already performed optimization. In particular, in our experiments rather low dimensional problems from UCI library can be well solved by starting from random points sampled uniformly from the unit sphere. In contrast, when number of dimensions is significantly higher and the optimization problem is harder it is more valuable to initialize the weights vector by running optimization from balanced SVM solution (for example, with $C=1$). Such an approach is further used in the last section of the evaluation.

We have shown how \RMC{} behaves when treated as non-parametric model. However, similarly to the $C$ parameter in SVM formulation, we can control the strength of the regularization. 
In Table~\ref{tab:UCI_tuning} one can find WAC scores for \RMC{} (with fitted $\gamma$) as compared to SVM and balanced SVM (with fitted $C$). 
\begin{table}[h]
\begin{center}
\begin{tabular}{lrrr}
\hline
dataset	&	\RMC{} 	& SVM	& SVM-B \\
\hline

australian 	& \textbf{0.868} [1.0] & 0.862 & 0.862 \\
breast cancer 	& \textbf{0.979} [1.0] & 0.969 & 0.972 \\
diabetes 	& \textbf{0.758} [1.0] & 0.727 & 0.747 \\
fourclass 	& \textbf{0.843} [4.0] & 0.720 & 0.727 \\
german number 	& \textbf{0.726} [1.1] & 0.691 & 0.722 \\
heart 		& 0.836 [1.0] & 0.837 & \textbf{0.838} \\
ionosphere 	& 0.848 [1.0] & \textbf{0.862} & 0.860 \\
liver-disorders	& 0.658 [2.9] & \textbf{0.677} & 0.659 \\
sonar 		& \textbf{0.791} [1.0] & 0.790 & 0.790 \\
splice 		& \textbf{0.810} [1.0] & \textbf{0.810} & \textbf{0.810} \\


\hline
\end{tabular} 
\end{center}
\caption{Comparision of 10-fold cross validation WAC scores for \RMC{} and SVM, SVM-balanced (SVM-B) with optimized parameters. Mean number of thresholds for \RMC{} is reported in square brackets}
\label{tab:UCI_tuning}
\end{table}
Obtained results resemble ones from the previous experiments, \RMC{} obtained similar results to the SVM, with some datasets showing superiority of the entropy based approach. In particular, in case of fourclass dataset one can see even bigger advantage of using multithreshold function over the simple linear classifier. It is also worth noting, that \RMC{} parameter has much more clear geometrical interpretation than SVM's parameter $C$ (which can be seen either as abstract weight of training errors or as an upper bound on the size of Lagrange multipliers). The parameter $\gamma$, or in general the formula for $V_A(v)$, gives the estimation of optimal kernel width in one dimensional projection of $A$ on $v$. There are many existing studies~\cite{Hammann2009P450SVM,zhang2013evaluation,subasi2013classification} and formulas for such objects, in particular it is possible to perform adaptive kernel width~\cite{van2003adaptive} where each point $x$ have its own kernel width $\sigma_x$. 

\subsection{Compounds activity prediction}

Final part of our evaluation was performed on cheminformatical data. The task is to predict whether a chemical compound is active, that is binds to a given protein. We used ten different proteins and corresponding sets of molecules with known (empirically tested) activity. Each compound was represented as the fixed length bit sequence using the SMARTS patterns \cite{padel}, which is one of the commonly used fingerprints (data representations) in such tasks \cite{smusz2013influence}. These gives us ten different binary datasets, summarized in Table~\ref{tab:summary_proteins}.

\begin{table}[H]
\begin{center}
\begin{tabular}{lrrrr}
\hline
protein	&	d 	& n & $|X_{\plusOne}|$ & $|X_{\minusOne}|$ \\
\hline
5-HT$_\text{2A}$ 	& 82 	& 2686 	& 1835 	& 851 \\
5-HT$_\text{6}$ 	& 109 	& 1831 	& 1490 	& 341 \\
5-HT$_\text{7}$ 	& 108 	& 1043 	& 704 	& 339 \\
cathepsin 		& 116 	& 1188 	& 245 	& 943 \\
D$_\text{2}$ 		& 137 	& 6215 	& 3342 	& 2873 \\
hERG 			& 130 	& 4928 	& 496 	& 4432 \\
HIV integrase 		& 130 	& 1015 	& 101 	& 914 \\
HIV protease 		& 134 	& 4052 	& 3155 	& 897 \\
M$_\text{1}$ 		& 123 	& 1697 	& 759 	& 938 \\
SERT 			& 129 	& 5231 	& 3559 	& 1672 \\

\hline
\end{tabular} 
\end{center}
\caption{Summary of cheminformatics datasets used in tests. SubFP~\cite{smusz2013influence} is used for molecules representation.}
\label{tab:summary_proteins}
\end{table} 

Similarly to the previous experiments, we compare \RMC{} with fitted $\gamma$ parameter ($\gamma \in \{ 0.1,0.2,...,1.1 \}$) with balanced linear SVM with fitted $C$. Greedy gradient optimization is performed from the set of starting points consisting of random points (uniformly selected from the unit sphere), solution of balanced SVM with $C=1$, and solution of perceptron. The model with highest $D_{CS}$ value is selected. Conducted experiments, summarized in Table~\ref{tab:reults_proteins}, show that our method is a competetive model for this kind of data. It is clear that for some proteins (like 5-HT$_\text{2A}$ or cathepsin, see Figure \ref{fig:mlc}) the internal data geometry can be better exploited using multithreshold linear classifier. Namely such model can detect, contrary to 
single threshold linear model, some underrepresented classes of active molecules which can be of high importance in the the search for new proteins' ligands.
\begin{table}[H]
\begin{center}
\begin{tabular}{lrr|rr}
\hline
protein					& MCC & 	& WAC   &       			\\
					& \RMC{} & SVM-B & \RMC{}	& SVM-B 					\\
\hline
5-HT$_\text{2A}$ 	& \textbf{0.434} [2.8] & 0.379& \textbf{0.725} [2.8]	& 0.703					 \\ 
5-HT$_\text{6}$ 	& \textbf{0.604} [3.0] & 0.593	& \textbf{0.835} [3.0] 	& 0.834 			 \\
5-HT$_\text{7}$ 	& \textbf{0.464} [9.8] & 0.435	& \textbf{0.735} [9.8] 	& 0.723				 \\
cathepsin 		& \textbf{0.530} [1.0] & 0.476		  & \textbf{0.796} [1.0] 	& 0.779		\\	
D$_\text{2}$ 		& 0.441 [1.0] & \textbf{0.442}		 & 0.720 [1.2] 		& \textbf{0.721} 	\\
hERG 			& \textbf{0.320} [1.5] & 0.304		      & \textbf{0.740} [1.3] 	& 0.738	\\		
HIV integrase 		& \textbf{0.543} [4.6] & 0.515		& 0.834 [1.1] 		& \textbf{0.835}	 \\
HIV protease 		& \textbf{0.501} [1.0] & 0.493		 & \textbf{0.782} [1.0]	& \textbf{0.782} 	\\
M$_\text{1}$ 		& \textbf{0.536} [3.6] & 0.532		 & \textbf{0.769} [3.6] 	& 0.766			\\
SERT 			& \textbf{0.439} [1.0] & 0.438		      & \textbf{0.734} [1.4]	& 0.733	\\		

\hline
\end{tabular} 
\end{center}
\caption{Summary of results for cheminformatics data.}
\label{tab:reults_proteins}
\end{table} 

By examining scores for some particular folds we can see that despite similarities, \RMC{} and SVM have difficulties in classifying
different datasets. In particular one can see on Figure~\ref{fig:folds_proteins} that for 5-HT$_\text{2A}$ dataset, third fold was the hardest one (in terms of MCC) for SVM while the same data seemed as easy for \RMC{} as the first or the second one. This shows that resulting models are indeed different.
\begin{figure}[h]
 \begin{center}
 \includegraphics[width=\textwidth]{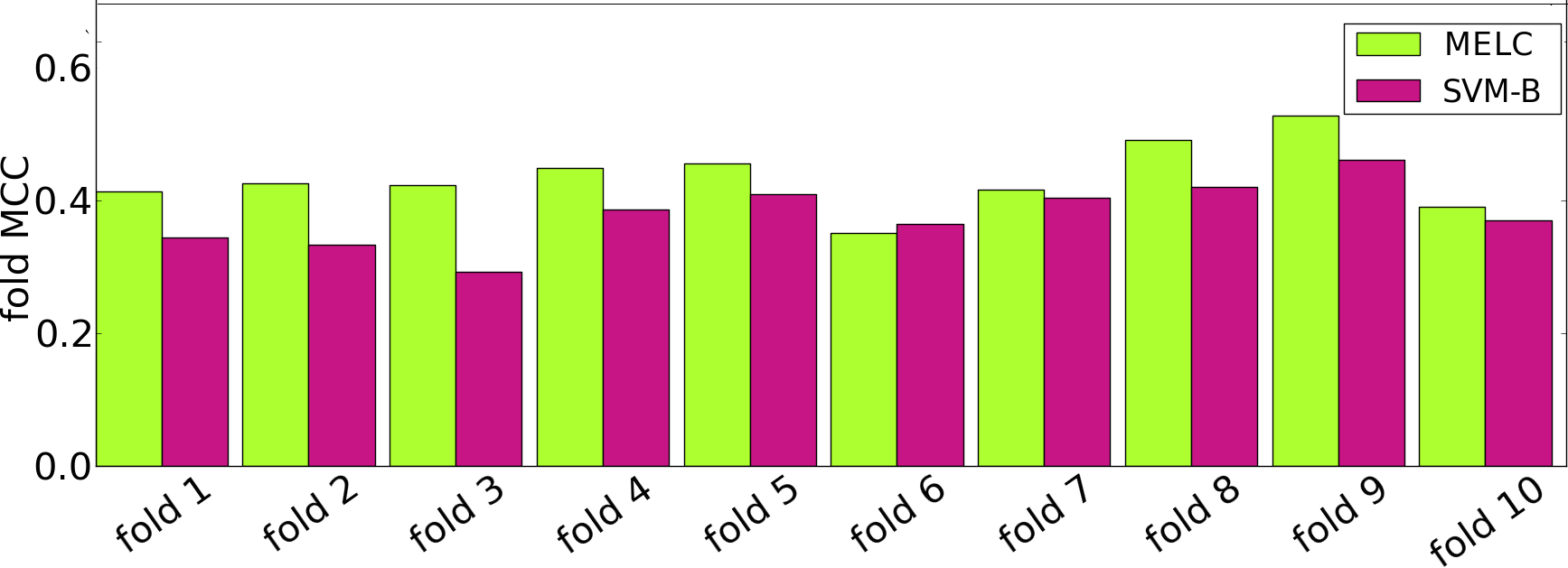}
 \end{center}
 \caption{Matthew's Correlation Coefficient for each fold of 5-HT$_\text{2A}$ dataset using \RMC{} and balanced SVM.}
 \label{fig:folds_proteins}
\end{figure}

Naturally, better overall results could be obtained using kernelized SVM (with RBF kernel), although it would lead to creation of very complex models (number of support vectors for these datasets varies between 1000 and 2000). As the result, constructed classifier is big and slow (it consists of about 100,000 numbers, and requires thousands of $\exp$ evaluations), while at the same time 
\RMC{} builds very light model, consisting of about $d+3$ numbers ($d+10$ in case of 5-HT$_\text{7}$). It is an important factor, as speed of a resulting model is an important aspect for actual applications of compounds activity classifiers, which should be able to process huge databases of possible molecules. 

We have also checked how many iterations of \textit{binsearch} is required to build a good (close to the density based) $k$-threshold linear classifier. For considered datasets, performing just one iteration (placing the threshold in the middle between two points projections) led to very similar results (see Table~\ref{tab:iters_proteins}). Performing five iterations led to exactly the same scores as achieved with density based model.

\begin{table}[H]
\begin{center}
\begin{tabular}{lrrrrr}
\hline
protein			& $i=1$ & $i=2$& $i=3$ & $i=4$ & $i=5$ \\
\hline
5-HT$_\text{2A}$ 	& 0.001 & 0	& 0	& 0 	& 0\\
5-HT$_\text{6}$ 	& 0.001 & 0	& 0	& 0 	& 0\\
5-HT$_\text{7}$ 	& 0.004 & 0.003 & 0.003	& 0.001	& 0\\
cathepsin 		& 0	& 0	& 0	& 0 	& 0\\
D$_\text{2}$ 		& 0.001	& 0.001	& 0.001	& 0 	& 0\\
hERG 			& 0.001 & 0.001 & 0.001	& 0.001	& 0\\
HIV integrase 		& 0.003 & 0.003 & 0.002	& 0.002	& 0\\
HIV protease 		&0	& 0	& 0	& 0 	& 0\\
M$_\text{1}$ 		& 0	& 0.001	& 0	& 0 	& 0\\
SERT 			& 0 	& 0 	& 0	& 0 	& 0\\

\hline
\end{tabular} 
\end{center}
\caption{Differences between MCC scores of density based classifier and $k$-threshold after $i$ iterations of binsearch. Differences for WAC were even smaller}
\label{tab:iters_proteins}
\end{table} 

To sum up, the evaluation on the real, cheminformatics dataset lead to the following conclusions regarding proposed model:
\begin{itemize}
 \item obtained results are (in most cases) better than those obtained by balanced SVM,
 \item resulting model has the same complexity as linear models (and rows of magnitude smaller than kernelized ones),
 \item internal data geometry of chemical compounds can be better exploited using multithreshold model,
 \item multithreshold structure might lead to detection of underrepresented active/inactive compounds families,
 \item just a few iterations of binsearch are requried to convert a density based method to actual multithreshold function. 
\end{itemize}

\section{Conclusions}

In this paper we presented a novel multithreshold classification method based on Renyi's quadratic entropy. Proposed model is based on search for the best linear projection on $\R$ in terms of Cauchy-Schwarz divergence of kernel estimation of the data projection. We showed its theoretical justification and properties, including scale invariance and relations to the 
largest margin SVM classifier.  

We proposed a simple, gradient based constrained optimization method for the construction of density-based classifier. However, it remains an open issue how to efficiently optimize it, as outlined approach has high computational complexity. We also showed how such a classifier can be efficiently converted to the $k$-threshold linear classifier.

During evaluation we studied how proposed model behaves on UCI binary datasets as well as real data coming from cheminformatics. In most cases, \RMC{} behaved better than balanced SVM in terms of balanced evaluation measures (WAC and MCC). We also investigated existance of correlation between our criterion and the generalization error. Obtained results support our claim that proposed method performs structural risk minimization.


%



  \section*{Acknowledgments}

This work was partially founded by National Science Centre Poland grant no. 2013/09/N/ST6/03015.

The authors would like to thank Igor Podolak, Phd for his invaluable contribution to our work, discussions, suggestions and criticism. We would also like to thank Sabina Smusz, MSc from Institute of Pharmacology, Polish Institute of Science for providing access to the cheminformatics data and sharing knowledge regarding compounds' activity prediction. Finally, we would like to thank Daniel Wilczak, Phd for access to the Fermi supercomputer which made the numerous experiments possible.




\bibliographystyle{IEEEtran}

\bibliography{bibs}

\begin{thebibliography}{10}
\providecommand{\url}[1]{#1}
\csname url@samestyle\endcsname
\providecommand{\newblock}{\relax}
\providecommand{\bibinfo}[2]{#2}
\providecommand{\BIBentrySTDinterwordspacing}{\spaceskip=0pt\relax}
\providecommand{\BIBentryALTinterwordstretchfactor}{4}
\providecommand{\BIBentryALTinterwordspacing}{\spaceskip=\fontdimen2\font plus
\BIBentryALTinterwordstretchfactor\fontdimen3\font minus
  \fontdimen4\font\relax}
\providecommand{\BIBforeignlanguage}[2]{{%
\expandafter\ifx\csname l@#1\endcsname\relax
\typeout{** WARNING: IEEEtran.bst: No hyphenation pattern has been}%
\typeout{** loaded for the language `#1'. Using the pattern for}%
\typeout{** the default language instead.}%
\else
\language=\csname l@#1\endcsname
\fi
#2}}
\providecommand{\BIBdecl}{\relax}
\BIBdecl

\bibitem{huang2011extreme}
G.-B. Huang, D.~H. Wang, and Y.~Lan, ``Extreme learning machines: a survey,''
  \emph{International Journal of Machine Learning and Cybernetics}, vol.~2,
  no.~2, pp. 107--122, 2011.

\bibitem{hinton2006fast}
G.~E. Hinton, S.~Osindero, and Y.-W. Teh, ``A fast learning algorithm for deep
  belief nets,'' \emph{Neural computation}, vol.~18, no.~7, pp. 1527--1554,
  2006.

\bibitem{cortes1995support}
C.~Cortes and V.~Vapnik, ``Support-vector networks,'' \emph{Machine learning},
  vol.~20, no.~3, pp. 273--297, 1995.

\bibitem{cao1994comparative}
R.~Cao, A.~Cuevas, and W.~Gonzalez~Manteiga, ``A comparative study of several
  smoothing methods in density estimation,'' \emph{Computational Statistics \&
  Data Analysis}, vol.~17, no.~2, pp. 153--176, 1994.

\bibitem{principe2000information}
J.~C. Principe, D.~Xu, and J.~Fisher, ``Information theoretic learning,''
  \emph{Unsupervised adaptive filtering}, vol.~1, pp. 265--319, 2000.

\bibitem{silverman1986density}
B.~W. Silverman, \emph{Density estimation for statistics and data
  analysis}.\hskip 1em plus 0.5em minus 0.4em\relax CRC press, 1986, vol.~26.

\bibitem{anthony2004generalization}
M.~Anthony, ``Generalization error bounds for threshold decision lists,''
  \emph{The Journal of Machine Learning Research}, vol.~5, pp. 189--217, 2004.

\bibitem{smusz2013influence}
S.~Smusz, R.~Kurczab, and A.~J. Bojarski, ``The influence of the inactives
  subset generation on the performance of machine learning methods,''
  \emph{Journal of cheminformatics}, vol.~5, no.~1, pp. 1--8, 2013.

\bibitem{takiyama1978multiple}
R.~Takiyama, ``Multiple threshold perceptron,'' \emph{Pattern Recognition},
  vol.~10, no.~1, pp. 27--30, 1978.

\bibitem{olafsson1988capacity}
S.~Olafsson and Y.~S. Abu-Mostafa, ``The capacity of multilevel threshold
  functions,'' \emph{IEEE Transactions on Pattern Analysis and Machine
  Intelligence}, vol.~10, no.~2, pp. 277--281, 1988.

\bibitem{anthony2003learning}
M.~Anthony, \emph{Learning multivalued multithreshold functions}.\hskip 1em
  plus 0.5em minus 0.4em\relax Citeseer, 2003.

\bibitem{huang2008maxi}
K.~Huang, H.~Yang, I.~King, and M.~R. Lyu, ``Maxi--min margin machine: learning
  large margin classifiers locally and globally,'' \emph{Neural Networks, IEEE
  Transactions on}, vol.~19, no.~2, pp. 260--272, 2008.

\bibitem{freund1999large}
Y.~Freund and R.~E. Schapire, ``Large margin classification using the
  perceptron algorithm,'' \emph{Machine learning}, vol.~37, no.~3, pp.
  277--296, 1999.

\bibitem{tipping2003relevance}
M.~E. Tipping, ``Sparse bayesian learning and the relevance vector machine,''
  \emph{The journal of machine learning research}, vol.~1, pp. 211--244, 2001.

\bibitem{principe-class}
J.~C. Principe, \emph{Information theoretic learning}.\hskip 1em plus 0.5em
  minus 0.4em\relax Springer, 2000.

\bibitem{principe-clust}
J.~C. Principe, R.~Jenssen, and S.~Rao, ``Clustering with itl principles,'' in
  \emph{Information theoretic learning}.\hskip 1em plus 0.5em minus 0.4em\relax
  Springer, 2000, pp. 263--298.

\bibitem{principe-self}
J.~C. Principe, S.~Rao, D.~Erdogmus, D.~Xu, and K.~I. Hild, ``Self-organizing
  itl principles for unsupervised learning,'' in \emph{Information theoretic
  learning}.\hskip 1em plus 0.5em minus 0.4em\relax Springer, 2000, pp.
  263--298.

\bibitem{santos2004error}
J.~M. Santos, L.~A. Alexandre, and J.~M. de~S{\'a}, ``The error entropy
  minimization algorithm for neural network classification,'' in
  \emph{International Conference on Recent Advances in Soft Computing}.\hskip
  1em plus 0.5em minus 0.4em\relax Citeseer, 2004, pp. 92--97.

\bibitem{breiman2001random}
L.~Breiman, ``Random forests,'' \emph{Machine learning}, vol.~45, no.~1, pp.
  5--32, 2001.

\bibitem{statystyka}
N.~Timm, \emph{Applied multivariate Analysis}.\hskip 1em plus 0.5em minus
  0.4em\relax Springer Text in Statistics, 2002.

\bibitem{drineas2005nystrom}
P.~Drineas and M.~W. Mahoney, ``On the nystr{\"o}m method for approximating a
  gram matrix for improved kernel-based learning,'' \emph{The Journal of
  Machine Learning Research}, vol.~6, pp. 2153--2175, 2005.

\bibitem{huang2006extreme}
G.-B. Huang, Q.-Y. Zhu, and C.-K. Siew, ``Extreme learning machine: theory and
  applications,'' \emph{Neurocomputing}, vol.~70, no.~1, pp. 489--501, 2006.

\bibitem{hegde2007random}
C.~Hegde, M.~B. Wakin, and R.~G. Baraniuk, ``Random projections for manifold
  learning.'' in \emph{NIPS}, vol.~7, 2007, p.~59.

\bibitem{haykin2009neural}
S.~S. Haykin, \emph{Neural networks and learning machines}.\hskip 1em plus
  0.5em minus 0.4em\relax Pearson Education Upper Saddle River, 2009, vol.~3.

\bibitem{karlsson2005beyond}
B.~Karlsson, \emph{Beyond the C++ standard library: an introduction to
  boost}.\hskip 1em plus 0.5em minus 0.4em\relax Pearson Education, 2005.

\bibitem{asuncion2007uci}
A.~Asuncion and D.~Newman, ``Uci machine learning repository,'' 2007.

\bibitem{pedregosa2011scikit}
F.~Pedregosa, G.~Varoquaux, A.~Gramfort, V.~Michel, B.~Thirion, O.~Grisel,
  M.~Blondel, P.~Prettenhofer, R.~Weiss, V.~Dubourg \emph{et~al.},
  ``Scikit-learn: Machine learning in python,'' \emph{The Journal of Machine
  Learning Research}, vol.~12, pp. 2825--2830, 2011.

\bibitem{chang2011libsvm}
C.-C. Chang and C.-J. Lin, ``Libsvm: a library for support vector machines,''
  \emph{ACM Transactions on Intelligent Systems and Technology (TIST)}, vol.~2,
  no.~3, p.~27, 2011.

\bibitem{Hammann2009P450SVM}
F.~Hammann, H.~Gutmann, U.~Baumann, C.~Helma, and J.~Drewe, ``{Classification
  of Cytochrome P 450 Activities Using Machine Learning Methods},''
  \emph{Molecular Pharmaceutics}, vol.~33, no.~1, pp. 796--801, 2009.

\bibitem{zhang2013evaluation}
X.~Zhang, X.~Liu, and Z.~J. Wang, ``Evaluation of a set of new orf kernel
  functions of svm for speech recognition,'' \emph{Engineering Applications of
  Artificial Intelligence}, vol.~26, no.~10, pp. 2574--2580, 2013.

\bibitem{subasi2013classification}
A.~Subasi, ``Classification of emg signals using pso optimized svm for
  diagnosis of neuromuscular disorders,'' \emph{Computers in biology and
  medicine}, vol.~43, no.~5, pp. 576--586, 2013.

\bibitem{van2003adaptive}
P.~Van~Kerm, ``Adaptive kernel density estimation,'' \emph{Stata Journal},
  vol.~3, no.~2, pp. 148--156, 2003.

\bibitem{padel}
C.~W. Yap, ``Padel-descriptor: An open source software to calculate molecular
  descriptors and fingerprints,'' \emph{Journal of Computational Chemistry},
  vol.~32, no.~7, pp. 1466--1474, 2011.

\end{thebibliography}
\end{document}